\documentclass[11pt]{article}

\usepackage{fullpage,times,url,bm}

\usepackage{amsthm,amsfonts,amsmath,amssymb,epsfig,color,float,graphicx,verbatim}
\usepackage{algorithm,algorithmic}
\usepackage{bbm}
\usepackage{natbib}

\usepackage{graphicx}
\usepackage{subcaption}

\usepackage{hyperref}
\hypersetup{
	colorlinks   = true, 
	urlcolor     = blue, 
	linkcolor    = blue, 
	citecolor   = black 
}

\newtheorem{theorem}{Theorem}[section]
\newtheorem{proposition}{Proposition}[section]
\newtheorem{lemma}{Lemma}[section]
\newtheorem{corollary}{Corollary}[section]

\newtheorem{remark}{Remark}[section]

\usepackage{amssymb}

\usepackage{dsfont}

\newcommand{\stam}[1]{}

\usepackage{mathtools}

\newtheorem{assumption}[theorem]{Assumption}

\newcommand{\be}{\mathbf{e}}
\newcommand{\bx}{\mathbf{x}}
\newcommand{\bw}{\mathbf{w}}

\newcommand{\bu}{\mathbf{u}}
\newcommand{\bv}{\mathbf{v}}

\newcommand{\by}{\mathbf{y}}

\newcommand{\btheta}{{\boldsymbol{\theta}}}

\newcommand{\Xcal}{\mathcal{X}}

\newcommand{\cb}{{\cal B}}

\newcommand{\cq}{{\cal Q}}

\newcommand{\ci}{{\cal I}}

\newcommand{\cw}{{\cal W}}
\newcommand{\cl}{{\cal L}}

\newcommand{\cu}{{\cal U}}

\newcommand{\cs}{{\cal S}}

\newcommand{\calr}{{\cal R}}

\DeclareMathOperator*{\argmin}{argmin}

\newcommand{\reals}{{\mathbb R}}
\newcommand{\nat}{{\mathbb N}}

\newcommand{\zero}{{\mathbf{0}}}

\newcommand{\diag}{\mathrm{diag}}

\newcommand{\inner}[1]{\langle #1 \rangle}
\newcommand{\norm}[1]{\left\|#1\right\|}

\title{Implicit Regularization in ReLU Networks with the Square Loss}

\author{Gal Vardi \qquad Ohad Shamir\\
Weizmann Institute of Science\\ 
\texttt{\{gal.vardi,ohad.shamir\}@weizmann.ac.il}
}
\date{}

\begin{document}

\maketitle

\begin{abstract}
Understanding the implicit regularization (or implicit bias) of gradient descent has recently been a very active research area. However, the implicit regularization in nonlinear neural networks is still poorly understood, especially for regression losses such as the square loss. Perhaps surprisingly, we prove that even for a single ReLU neuron, it is impossible to characterize the implicit regularization with the square loss by any explicit function of the model parameters (although on the positive side, we show it can be characterized approximately). For one hidden-layer networks, we prove a similar result, where in general it is impossible to  characterize implicit regularization properties in this manner, except for the ``balancedness'' property identified in \cite{du2018algorithmic}. Our results suggest that a more general framework than the one considered so far may be needed to understand implicit regularization for nonlinear predictors, and provides some clues on what this framework should be.
\end{abstract}

\section{Introduction}

A major open question in the theory of deep learning is how neural networks generalize even when trained without any explicit regularization, and when there are far more learnable parameters than training examples. In such an underdetermined optimization problem, there are many global minima with zero training loss, and gradient descent seems to prefer solutions that generalize well (see \cite{zhang2016understanding}). Hence, it is believed that gradient descent induces an {\em implicit regularization} (or implicit bias) \citep{neyshabur2014search,neyshabur2017exploring}, and characterizing this regularization/bias has been a subject for an extensive research in recent years. 

The focus in the existing research is on finding a regularization function $\calr(\btheta)$, where $\btheta$ are the parameters of the model, such that if we apply gradient descent on the average loss, then it converges in some sense to a global optimum that minimizes $\calr$. 
Thus, the function $\calr$ determines to which global optimum gradient descent converges. For example, it is known that for linear regression, gradient descent on the average square loss converges to the zero-loss solution with minimal $\ell_2$ norm \citep{zhang2016understanding}. For linear classification and exponentially-tailed losses (such as the logistic or exponential losses), gradient descent on linearly separable data converges in direction to the minimal $\ell_2$ norm predictor which attains a fixed positive margin over the data, also known as the max-margin solution (cf. \cite{soudry2018implicit}). Thus, in both cases $\calr$ is essentially the $\ell_2$ norm. 

Once we move beyond simple linear classification and regression, the situation gets more complicated. This is especially so when considering  the square loss (a setting which we focus on in this paper), where existing works still focus on predictors which express linear functions\footnote{See the related work section below for discussion of results on classification losses, which are relatively easier to tackle.}. For example, there has been much effort to characterize the implicit regularization of gradient descent in matrix factorization. This problem corresponds to training a depth-2 linear neural network, and is considered a well-studied test-bed for studying implicit regularization in deep learning. \cite{gunasekar2018implicit} conjectured that the implicit regularization in matrix factorization is the nuclear norm, and proved it for some restricted cases. 
This conjecture was further studied in 
a string of works (e.g., \cite{belabbas2020implicit,arora2019implicit,razin2020implicit}) providing positive and negative evidence, and was formally refuted by \cite{li2020towards}. 
Focusing on a more specific setting, \cite{woodworth2020kernel} recently studied the implicit regularization in {\em diagonal linear neural networks}, namely, linear networks where the weight matrices have a diagonal structure. They find the regularization function and show how it interpolates between the $\ell_1$ and the $\ell_2$ norm depending on the initialization scale. This result was further generalized in \cite{yun2020unifying}.

In this paper, we study implicit regularization with the square loss in nonlinear neural networks, when trained using gradient descent with infinitesimal step sizes (a.k.a. gradient flow). Perhaps surprisingly, we show that already for very simple such networks (involving a single ReLU neuron, or one thin hidden layer), the implicit regularization \emph{cannot} be expressed by any explicit function of the model parameters. In other words, the model used to capture the implicit regularization in previous works cannot be used to capture the implicit regularization of nonlinear models in general (at least with respect to the square loss). However, on the positive side, we show that this implicit regularization can sometimes be captured by an explicit $\calr$, but only up to some constant approximation factor. 
In a bit more detail, our contributions are as follows:
\begin{itemize}
\item We start with single-neuron networks, namely, $\bx \mapsto \sigma(\inner{\bx,\bw})$. If the activation function $\sigma$ is strictly monotonic, then it is not hard to show that the implicit regularization is the $\ell_2$ norm. However, if $\sigma$ is the ReLU function, then we show that the implicit regularization is not expressible by any nontrivial function of $\bw$. A bit more precisely, suppose that $\calr:\reals^d \rightarrow \reals$ is such that if gradient flow converges to a zero-loss solution $\bw^*$, then $\bw^*$ is a zero-loss solution that minimizes $\calr$. We show that such a function $\calr$ must be constant in $\reals^d \setminus \{\zero\}$. Hence, perhaps surprisingly, the approach of precisely specifying the implicit bias of gradient flow via a regularization function is not feasible in single-neuron networks.
\item On the positive side, we show that while the implicit bias of a ReLU neuron cannot be specified exactly by a regularization function, it can be expressed approximately, within a factor of $2$, by the $\ell_2$ norm. That is, let $\bw^*$ be a zero-loss solution with a minimal $\ell_2$ norm, and assume that gradient flow converges to some point $\bw(\infty)$, then $\norm{\bw(\infty)}_2 \leq 2 \norm{\bw^*}_2$. Assuming $\norm{\bw^*}$ is not too large, such a bound on the $\ell_2$ norm can be used to derive good statistical generalization guarantees, via standard techniques (cf. \cite{shalev2014understanding}).
\item We extend our study to depth-$2$ ReLU networks. For such networks, an important implicit bias property shown in \cite{du2018algorithmic} is that gradient flow enforces the differences between square norms across different layers to remain invariant. Starting from a point close to zero, it implies that the magnitudes of all layers are automatically balanced. Thus, gradient flow induces a bias toward ``balanced'' layers. However, this bias is rather weak and is not related to properties that may allow us to obtain generalization guarantees, such as small norms, sparsity, or low rank. We consider networks with one hidden ReLU neuron, i.e., $\bx \mapsto v \cdot \sigma(\inner{\bx,\bw})$ (where $v,\bw$ are the trained parameters), 
which is the simplest case of a depth-$2$ ReLU network.
We show that the {\em only} bias which can be specified by a  function $\calr$ of the model parameters is the balancedness property: Namely, $\calr$ is constant in the set of parameters that satisfy this property.
Thus, no other implicit bias property can be expressed by such a function for single-hidden-neuron networks. This also has implications for general depth-$2$ ReLU networks: If $\calr$ is a function that specifies the implicit regularization of depth-$2$ networks, then in the special case of single-hidden-neuron networks, it does not induce bias except for enforcing the balancedness property. Hence, the implicit regularization function of depth-$2$ ReLU networks cannot be generally related to properties such as small norms or sparsity, which should apply also to single-hidden-neuron networks.
\end{itemize}

We again emphasize that since our results show the impossibility of characterizing the implicit regularization already for very simple networks, similar difficulties will be encountered in more general cases (such as ReLU networks of any depth and width). 

Overall, our results suggest that the implicit regularization for ReLU networks with the square loss may not be expressible by an explicit function $\calr$ of the model parameters. On the positive side, they also suggest how this can be overcome by changing the model: Instead of looking for such an $\calr$ which captures the bias exactly, we might try to capture it \emph{approximately}, as we did in the single ReLU neuron case. Another possible direction is to find a regularization function $\calr$ that is \emph{data-dependent} in some simple way\footnote{Otherwise, a trivial solution is to define $\calr(\btheta,\Xcal)$ (where $\Xcal$ is the dataset) to be $0$ if $\btheta$ are the parameters returned by gradient flow given $\Xcal$, and $1$ otherwise.}, and does not depend just on the model parameters. We believe these are both interesting directions for further research.

\subsection*{Related Work}

As we already discussed, implicit regularization in matrix factorization and linear neural networks was extensively studied, as a first step toward understanding implicit regularization in more complex models (see, e.g., \cite{gunasekar2018implicit,razin2020implicit,arora2019implicit,belabbas2020implicit,eftekhari2020implicit,li2018algorithmic,ma2018implicit,woodworth2020kernel,gidel2019implicit,li2020towards,yun2020unifying}).
In \cite{razin2020implicit} it was shown that gradient flow in matrix factorization may approach a global minimum at infinity rather than converging to a global minimum with a finite norm. This result suggests that the implicit regularization in matrix factorization may not be expressible by norms. The main conceptual differences from our paper are that we rule out all regularization functions, and that our result holds already if we assume that gradient flow converges to some global minimum with a finite norm.

\cite{oymak2019overparameterized} studied the implicit bias of gradient descent in certain nonlinear models. They showed that under some assumptions, gradient descent is guarranteed to converge to a zero-loss solution with a bounded $\ell_2$ norm. Their results have implications on the implicit regularization of single-neuron networks with a strictly monotonic activation function, which we discuss in Remark~\ref{rem:warm up}.
\cite{williams2019gradient} and \cite{jin2020implicit} studied the dynamics and implicit bias of gradient descent in  wide depth-$2$ ReLU networks with input dimension $1$.

The conceptually closest works to ours are \cite{dauber2020can,suggala2018connecting}, which show that variants of gradient descent on some (specially crafted) convex optimization problems do not converge to a closest Euclidean solution, or cannot be explained by any reasonable implicit regularization function. These works, as well as ours, consider the limitations of implicit regularization, but our results apply to a standard learning problem, namely, learning neural networks. Thus, while the results of \cite{dauber2020can,suggala2018connecting} demonstrate that there are learning problems that cannot be explained by any reasonable regularization function, we show that this phenomenon occurs already for simple ReLU networks.

The implicit bias of gradient descent in classification tasks is also widely studied. \cite{soudry2018implicit} showed that gradient descent on linearly-separable binary classification problem with an exponentially-tailed loss (e.g., the exponential loss and the logistic loss), converges to the maximum margin direction. This analysis was extended to other loss functions, tighter convergence rates, nonseparable data, and variants of gradient-based optimization algorithms  \citep{nacson2019convergence,ji2018risk,ji2020gradient,gunasekar2018characterizing,shamir2020gradient}.
The problem was also studied for more complex models, such as linear neural networks \citep{gunasekar2018bimplicit,ji2018gradient,moroshko2020implicit,yun2020unifying}, and neural networks with homogeneous activation functions \citep{lyu2019gradient,chizat2020implicit,xu2018will}.

\section{Preliminaries}


\textbf{Notations.}
We use bold-faced letters to denote vectors, e.g., $\bx=(x_1,\ldots,x_d)$. For $\bx \in \reals^d$ we denote by $\norm{\bx}$ the Euclidean norm.

\textbf{Single-neuron networks.}
Let $\{(\bx_i,y_i)\}_{i=1}^n$ be a training dataset, where for every $i$ we have $\bx_i \in \reals^d$ and $y_i \in \reals$. 
We consider empirical-loss minimization of a single-neuron network, with respect to the square loss. Thus, the objective is given by 
\begin{equation}
	\label{eq:problem}
	 \cl(\bw) := \frac{1}{2} \sum_{i=1}^n \left( \sigma(\inner{\bx_i,\bw}) - y_i \right)^2~,
\end{equation}
where $\sigma:\reals \rightarrow \reals$ is a nonlinear activation function. 
Let $X \in \reals^{n \times d}$ denote the data matrix, i.e., the rows of $X$ are $\bx_1^\top,\ldots,\bx_n^\top$, and let $\by=(y_1,\ldots,y_n)$. We have 
\[
\cl(\bw) = \frac{1}{2} \norm{\sigma(X\bw) -\by}^2~,
\]
where $\sigma$ is applied component-wise.
We assume that the data is realizable, that is, $\min_\bw \cl(\bw)=0$.
Moreover, we focus on settings where the network is {\em overparameterized}, in the sense that $\cl$ has multiple (or even infinitely many) global minima.

We analyze the implicit regularization of gradient flow on the objective given by Eq.~\ref{eq:problem}. 
This setting captures the behaviour of gradient descent with infinitesimally small step size. 
Let $\bw(t)$ be the trajectory of gradient flow, where $\bw(0)$ is the initial point.
The dynamics of $\bw(t)$ is given by the differential equation
\[
\dot{\bw}(t) := \frac{\mathop{d \bw(t)}}{\mathop{dt}} = -\nabla \cl(\bw(t))~.
\] 
If $\lim_{t \rightarrow \infty}\bw(t)$ exists then we denote it by $\bw(\infty)$.
We note that gradient flow is not guaranteed to converge to a global minimum (cf. \cite{yehudai2020learning}). In the overparameterized setting there can be infinitely many global minima, and we study to which one gradient flow converges, assuming that it converges to a global minimum.

\section{Warm up -- strictly monotonic activation functions}
\label{sec:strictly monotinic}

We first analyze the implicit regularization of gradient flow on the objective given by Eq.~\ref{eq:problem}, where the activation function $\sigma$ is strictly monotonic. In this case, we show that the implicit bias is simply the $\ell_2$ norm: Namely, the point $\bw$ closest to $\bw(0)$ on which $\cl(\bw)=0$. 

\begin{theorem}
\label{thm:strictly monotonic}
Consider gradient flow on the objective given by Eq.~\ref{eq:problem}, where $\sigma:\reals \rightarrow \reals$ is a strictly monotonic activation function. 
If $\bw(\infty)$ exists and $\cl(\bw(\infty))=0$, 
then $\bw(\infty) \in \argmin_{\bw} \norm{\bw - \bw(0)} \text{ s.t. } \cl(\bw)=0$.
\end{theorem}
\begin{proof}
Since $\sigma$ is strictly monotonic, then for every $1 \leq i \leq n$ we have $\sigma(\inner{\bx_i,\bw})=y_i$ iff $\inner{\bx_i,\bw} = \sigma^{-1}(y_i)$.	
The KKT optimality conditions for the problem 
\[
\min_\bw \norm{\bw - \bw(0)}^2 \text{ s.t. } \inner{\bx_i,\bw}=\sigma^{-1}(y_i), \; i=1,\ldots,n~,
\] 
are
\begin{align*}
\exists \lambda_1,\ldots,\lambda_n \text{ s.t. } \bw = \bw(0) + \sum_{i=1}^n\lambda_i \bx_i 
\\
\forall 1 \leq i \leq n, \; \inner{\bx_i,\bw}=\sigma^{-1}(y_i)~.
\end{align*}
Note that since the constraints are affine functions, then the conditions are both necessary and sufficient. 

The second condition holds in $\bw(\infty)$ since $\cl(\bw(\infty))=0$. 
Moreover, we have 
\[
\dot{\bw}(t) = -\nabla \cl(\bw(t)) = -\sum_{i=1}^n \left( \sigma(\bx_i^\top \bw(t))-y_i \right) \sigma'(\bx_i^\top \bw(t)) \bx_i~.
\]
Thus, $\dot{\bw}(t)$ is always spanned by $\bx_1,\ldots,\bx_n$, and hence the first condition also holds. 
\end{proof}

\begin{remark}
\label{rem:warm up}
	By \cite{oymak2019overparameterized},
	if $\sigma$ is continuously differentiable and there are positive constants $\gamma, \Gamma$ such that $0<\gamma \leq \sigma'(z) \leq \Gamma$ for every $z \in \reals$, 
	then $\bw(t)$ is guaranteed to converge to a global optimum $\bw(\infty)$ that is closest to $\bw(0)$. 
	In Theorem~\ref{thm:strictly monotonic}, our assumption on $\sigma$ is weaker, but we assume that $\bw(\infty)$ exists and $\cl(\bw(\infty))=0$ (which is a common assumption when considering implicit regularization).
\end{remark}

\section{The implicit regularization of a ReLU neuron is not expressible by $\bw$}
\label{sec:not function of w}

In this section we study the implicit regularization of gradient flow on the objective given by Eq.~\ref{eq:problem}, where $\bw(0) = \zero$ and $\sigma:\reals \rightarrow \reals$ is the ReLU function.
We impose the convention that even though the ReLU function is not differentiable at $0$, we take $\sigma'(0)>0$. This convention is required for analyzing gradient flow with $\bw(0)=\zero$, since if $\sigma'(0)=0$ then $\bw(t)$ would stay at $\zero$ indefinitely. 
For simplicty, in our proofs we assume that $\sigma'(0)=1$, but our results hold for every positive value.

Assume that the implicit regularization can be expressed as a function $\calr:\reals^d \rightarrow \reals$, namely, if gradient-flow converges to a global minimum $\bw^*$, then we have $\bw^* \in \argmin_\bw \calr(\bw) \text{ s.t. } \cl(\bw)=0$.
We will first show that unlike the case of strictly monotonic activations, the implicit bias of $\calr$ is \emph{not} the $\ell_2$ norm, and in fact, data-dependent in a nontrivial way. 
Then, we will extend this result, and show that $\calr$ is trivial, namely, constant in $\reals^d \setminus \{\zero\}$. 
Thus, the behavior of $\bw(\infty)$ has a complicated data-dependent behavior, and cannot be written down as minimizing some function dependent only on $\bw$, among all zero-loss solutions.

\subsection{$\calr$ is not the $\ell_2$ norm}

We show that the implicit regularization of a single ReLU neuron is not the $\ell_2$ norm.
Thus, the implicit regularization of ReLU is different than that of strictly monotonic activation functions considered in Section~\ref{sec:strictly monotinic}.
Specifically, in the following theorem we analyze the behavior of gradient flow for some family of inputs. The theorem considers a datasets of the form $(X_\gamma,\by) \in \reals^{3 \times 3} \times \reals^3$, where $\gamma$ is a parameter. We show that gradient flow starting from $\bw(0)=\zero$ with the input $(X_\gamma,\by)$, converges to a zero-loss solution which is not a minimal-norm zero-loss solution, and that the parameter $\gamma$ controls the ratio between $\norm{\bw(\infty)}$ and the minimal-norm of a zero-loss solution. 
Hence, the implicit regularization is not the $\ell_2$ norm. 
We will also use this theorem later in order to show additional results.

\begin{theorem}
\label{thm:not l2}
Let $\gamma \geq 0$, and let
\[
X_\gamma = \begin{bmatrix}
	3 & -1 & 0 \\
	4 & 2 & 0 \\
	0 & \gamma & \gamma
\end{bmatrix}~.
\]
Let $\alpha > 0$ and let $\by_\alpha = \alpha \cdot (16,18,0)^\top$.
Consider gradient flow with input $(X_\gamma,\by_\alpha)$, starting from $\bw(0)=\zero$, on the objective given by Eq.~\ref{eq:problem}, where $\sigma:\reals \rightarrow \reals$ is the ReLU function.
Let $\cw_\alpha = \{(5\alpha,-\alpha,s)^\top: s \leq 0\} \subseteq \reals^3$. Note that for every $\bw \in \cw_\alpha$ we have $\sigma(X_\gamma \bw)=\by_\alpha$, and thus $\cl(\bw) = 0$. 
Moreover, note that the global optimum with the minimal $\ell_2$ norm is $\alpha \cdot (5,-1,0)^\top$.
We have:
\begin{itemize}
	\item If $\gamma = 0$ then $\bw(\infty) = \alpha \cdot (5,-1,0)^\top$.
	\item If $\gamma = 1$ then $\bw(\infty) = \alpha \cdot (5,-1,s)^\top$ for some $s \in (-0.045,-0.035)$ independent of $\alpha$.
	\item If $\gamma = 2$ then $\bw(\infty) = \alpha \cdot (5,-1,s)^\top$ for some $s \in (-0.12,-0.1)$ independent of $\alpha$.
	\item If $\gamma = 5$ then $\bw(\infty) = \alpha \cdot (5,-1,s)^\top$ for some $s \in (-0.22,-0.2)$ independent of $\alpha$.
\end{itemize}
\end{theorem}
\begin{proof}[Proof sketch (for complete proof see Appendix~\ref{app:not function of w})]
The trajectory $\bw(t)$ is the solution of the initial value problem defined by the differential equation 
\begin{equation}
	\label{eq:3 examples diff}
	\dot{\bw}(t) 
	= - \nabla \cl(\bw(t)) 
	= - \sum_{i=1}^3 \left( \sigma(\bx_i^\top \bw(t)) - y_i \right) \sigma'(\bx_i^\top \bw(t)) \bx_i~,
\end{equation}
and the initial condition $\bw(0)=\zero$. Here, $\bx_i^\top$ are the rows of the matrix $X_\gamma$, and $y_i$ are the components of $\by_\alpha$.

In Figure~\ref{fig:sim}, we show the trajectories $\bw(t)$ of gradient descent (an empirical simulation), for $\alpha=1$ and $\gamma \in \{0,1,2,5,20\}$. 
The grey line denotes the set $\cw_\alpha$, and the hyperplane corresponds to $\bx_3^\top \bw = 0$ (where $\gamma>0$). 
Note that the trajectories for $\gamma>0$ consist of two parts. In the first part $\bx_3^\top \bw(t) \geq 0$, and in the second part $\bx_3^\top \bw(t) \leq 0$. We note that in the second part, the component $w_3(t)$ remains constant.
Indeed, by Eq.~\ref{eq:3 examples diff}, the change in $\bw(t)$ is spanned by $\bx_1,\bx_2,\bx_3$, and if $\bx_3^\top \bw(t) \leq 0$ then it is spanned only by $\bx_1,\bx_2$ (note that if $\bx_3^\top \bw(t) = 0$ then $\sigma(\bx_3^\top \bw(t))-y_3=0$). Since $\bx_1,\bx_2$ have $0$ as their third component, then $w_3(t)$ does not change. 
Intuitively, the parameter $\gamma$ controls the extent to which $w_3(t)$ decreases during the first part, before reaching the second part in which it stays constant. 
In the proof we analyze these trajectories formally.

For $\gamma=0$, we show that the solution to Eq.~\ref{eq:3 examples diff} is $\bw(t) = (w_1(t),w_2(t),0)$, where $\tilde{\bw}(t)=(w_1(t),w_2(t))$ is given by 
\begin{equation*}
	\tilde{\bw}(t) 
	= \tilde{X}^{-1} \tilde{\by} - \exp\left(-t \tilde{X}^\top \tilde{X}\right) \tilde{X}^{-1} \tilde{\by}~.
\end{equation*}
Here, $\tilde{X} \in \reals^{2 \times 2}$ is obtained from $X_0$ by omitting the third column and third row, $\tilde{\by}=(y_1,y_2)^\top$, and $\exp(\cdot)$ denotes the matrix exponential. Note that $\tilde{\bw}(\infty) = \tilde{X}^{-1} \tilde{\by} = \alpha \cdot (5,-1)^\top$.
Moreover, note that in order to show that this trajectory satisfies Eq.~\ref{eq:3 examples diff}, it suffices to show that $\bx_i^\top \bw(t) \geq 0$ for every $i \in \{1,2,3\}$ and $t \geq 0$, and that $\bw(t)$ satisfies the linear differential equation obtained form Eq.~\ref{eq:3 examples diff} by plugging in $\sigma(\bx_i^\top \bw(t))=\bx_i^\top \bw(t)$ and $\sigma'(\bx_i^\top \bw(t))=1$.

For $\gamma>0$, we show that the solution to Eq.~\ref{eq:3 examples diff} consists of two parts: The first part is where $t \in [0,t_1]$ for some $t_1>0$, and second part is where $t \in [t_1,\infty)$.
In the first part, we have 
\begin{equation*}
	\bw(t) = \alpha (5,-1,1)^\top - \exp\left(-t X_\gamma^\top X_\gamma\right) \alpha (5,-1,1)^\top~, 
\end{equation*}
and $\bx_i^\top \bw(t) \geq 0$ for every $i \in \{1,2,3\}$ and $t \in [0,t_1]$.
Then, in the second part, $w_3(t)$ remains constant, and $\tilde{\bw}(t)=(w_1(t),w_2(t))$ is given by 
\begin{align*}
	\tilde{\bw}(t) 
	= \alpha (5,-1)^\top + \exp\left(-(t-t_1) \tilde{X}_\gamma^\top \tilde{X}_\gamma\right) \left( \tilde{\bw}(t_1) - \alpha (5,-1)^\top \right)~.
\end{align*}
Here, $\tilde{X}_\gamma \in \reals^{2 \times 2}$ is obtained from $X_\gamma$ by omitting the third column and third row.
Note that $\tilde{\bw}(\infty) = \alpha \cdot (5,-1)^\top$. In this part, we have $\bx_i^\top \bw(t) \geq 0$ for $i \in \{1,2\}$ and $\bx_3^\top \bw(t) \leq 0$.

In the proof, we show for both parts, that the trajectories satisfy Eq.~\ref{eq:3 examples diff}. Moreover, we investigate the value of $w_3(\infty)=w_3(t_1)$ in order to obtain the required bounds for each $\gamma \in \{1,2,5\}$.
\end{proof}

\begin{figure}[t]
	\centering
	\includegraphics[scale=0.8]{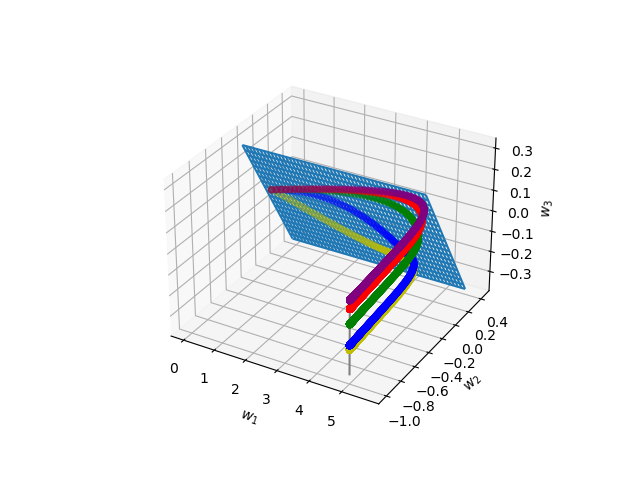}
	\caption{The trajectories for $\gamma=0$ (purple), $\gamma=1$ (red), $\gamma=2$ (green), $\gamma=5$ (blue), and $\gamma=20$ (yellow).} 
	\label{fig:sim}
\end{figure}



\stam{
Note that Theorem~\ref{thm:not l2} with $\alpha=\gamma=1$ implies that gradient flow converges to $\bw(\infty)=(5,-1,s)^\top$ for some $s<0$, while the zero-loss solution with the minimal $\ell_2$ norm is $(5,-1,0)^\top$. Hence, the implicit regularization of a ReLU neuron is not the $\ell_2$ norm.
}

\stam{
Hence, we have the following corollary.
\begin{corollary}
	The implicit regularization of a ReLU neuron is not the $\ell_2$ norm.
\end{corollary}
}

\begin{remark}
\label{rem:gamma controls distance}
	Theorem~\ref{thm:not l2} shows that by modifying the parameter $\gamma$ we can control the 
	ratio $\frac{\norm{\bw(\infty)}}{\norm{w^*}}$, where $\bw^*$ is a minimal-norm zero-loss solution.
	In Section~\ref{sec:positive result} we will show that this 
	ratio
	is bounded.
\end{remark}

\subsection{$\calr$ is trivial}
\label{sec:R trivial}

We now prove the following theorem, which implies that $\calr$ is trivial and does not induce an implicit bias.

\begin{theorem}
\label{thm:R trivial}
	Consider gradient flow starting from $\bw(0)=\zero$, on the objective given by Eq.~\ref{eq:problem}, where $\sigma:\reals \rightarrow \reals$ is the ReLU function.
	Let $\calr: \reals^3 \rightarrow \reals$, such that for every input $(X,\by) \in \reals^{3 \times 3} \times \reals^3$ where gradient flow converges to $\bw(\infty)$ with $\cl(\bw(\infty))=0$, we have $\bw(\infty) \in \argmin_\bw \calr(\bw) \text{ s.t. } \cl(\bw) = 0$. Then, $\calr$ is constant in $\reals^3 \setminus \{\zero\}$.
\end{theorem}

Let $X \in \reals^{3 \times 3}$ and $\by \in \reals^3$. Let $\bx_1^\top,\bx_2^\top,\bx_3^\top$ be the rows of $X$. Suppose that given input $(X,\by)$, gradient flow with $\bw(0)=\zero$ converges to $\bw(\infty)$. Note that by rotating $\bx_1,\bx_2,\bx_3$, the trajectory $\bw(t)$ is rotated accordingly 

Thus, for every rotation matrix $M \in SO(3)$, given input $(XM^\top,\by)$ gradient flow converges to $M \bw(\infty)$. 

Let $\alpha>0$. By Theorem~\ref{thm:not l2}, all points in $\cw_\alpha$ are global minima for all inputs $\{(X_\gamma,\by_\alpha): \gamma \geq 0\}$, but gradient-flow converges to different points $\bw^\gamma \in \cw_\alpha$, depending on $\gamma$. Hence, we have $\calr(\bw^\gamma) = \min_{\bw \in \cw_\alpha}\calr(\bw)$.
Thus, Theorem~\ref{thm:not l2} implies that $\calr(\bw^0)=\calr(\bw^1)=\calr(\bw^2)=\calr(\bw^5) = \min_{\bw \in \cw_\alpha}\calr(\bw)$. We use this property in order to prove Theorem~\ref{thm:R trivial}.

Let $M \in SO(3)$. We denote $M \cw_\alpha = \{M \bw: \bw \in \cw_\alpha\}$.
By changing the input $X_\gamma$ in Theorem~\ref{thm:not l2} to $X_\gamma M^\top$, we deduce that $\calr(M \bw^0) = \calr(M \bw^1) = \min_{\bw \in M \cw_\alpha}\calr(\bw)$. 
Indeed,  all points in $M \cw_\alpha$ are global minima for all inputs $\{(X_\gamma M^\top,\by_\alpha): \gamma \geq 0\}$, but gradient-flow converges to different points $M \bw^\gamma \in M \cw_\alpha$, depending on $\gamma$.
Let $c>0$ be such that $\bw^1 = \alpha \cdot (5,-1,-c)$.
For every $\bw \in \reals^3$ with $\norm{\bw} = \norm{\bw^0} = \norm{\alpha \cdot (5,-1,0)} = \alpha \sqrt{26}$, and $\bv \in \reals^3$ such that $\inner{\bw,\bv} = 0$ and $\norm{\bv} = \alpha \cdot c$, by choosing an appropriate rotation $M$ we obtain $ \calr(\bw) = \calr(\bw+\bv) = \min_{\beta \geq 0}\calr(\bw + \beta \bv)$.
Finally, as we show in the following lemma,
this property implies that $\calr$ is constant over $\reals^3 \setminus \{\zero\}$, and thus we complete the proof of Theorem~\ref{thm:R trivial}.

\begin{lemma}
	\label{lemma:R constant}
	Let $q,c > 0$ be constants.
	Let $d \geq 3$, and let $\calr:\reals^d \rightarrow \reals$ be a function such that for every $\alpha>0$, vector $\bw \in \reals^d$ with $\norm{\bw} = \alpha q$, and vector $\bv \in \reals^d$ such that $\norm{\bv} = \alpha c$ and $\inner{\bw,\bv} = 0$, we have $\calr(\bw) = \calr(\bw+\bv) = \min_{\beta \geq 0}\calr(\bw + \beta \bv)$.
	Then, $\calr(\bw) = \calr(\bw')$ for every $\bw,\bw' \in \reals^d \setminus \{\zero\}$. 
\end{lemma}
\begin{proof}[Proof sketch (for complete proof see Appendix~\ref{app:R constant})]
	First, we show that $\calr$ is radial, namely, there exists a function $f: \reals_+ \rightarrow \reals$ such that $\calr(\bw)=f(\norm{\bw})$ for every $\bw \in \reals^d$. Assume that $\bw,\bw' \in \reals^d \setminus \{\zero\}$ are such that $\norm{\bw}=\norm{\bw'}=\alpha q$ for some $\alpha>0$, and $\norm{\bw-\bw'}$ is sufficiently small. We show that there exists $\bu^* \in \reals^d$ such that $\bu^* = \bw + \bv = \bw' + \bv'$ for some $\bv,\bv'$ with $\norm{\bv}=\norm{\bv'}=\alpha c$ and $\inner{\bw,\bv}=\inner{\bw',\bv'}=0$. By our assumption on $\calr$, it implies that $\calr(\bw) = \calr(\bu^*) = \calr(\bw')$, and hence $\calr$ is radial. 
	
	Thus, it suffices to show that $f(r)=f(r')$ for every $0<r<r'$.
	Using our assumption on $\calr$, we show that for every $m \in \nat$ we have $f(r)=f\left(r  \left(1+ \frac{c^2}{q^2}\right)^{m/2} \right)$. The assumption on $\calr$ also implies that $f(r) = \min_{r'' \geq r}f(r'') \leq f(r')$, and that $f(r') \leq f\left(r  \left(1+ \frac{c^2}{q^2}\right)^{m/2} \right)$ for a sufficiently large $m$. Therefore, we have $f(r) \leq f(r') \leq f\left(r  \left(1+ \frac{c^2}{q^2}\right)^{m/2} \right) = f(r)$, and hence $f(r)=f(r')$ as required.
\end{proof}

\begin{remark}[$\calr$ is trivial also for $d>3$] 
	Theorem~\ref{thm:R trivial} implies that the implicit regularization $\calr:\reals^d \rightarrow \reals$ is trivial for $d=3$. By padding the inputs $\bx_1,\bx_2,\bx_3$ from Theorem~\ref{thm:not l2} with zeros, we can obtain a $d$-dimensional version of Theorem~\ref{thm:not l2} for $d>3$. Then, using the same arguments as above, it implies that Theorem~\ref{thm:R trivial} holds for every $d \geq 3$.
\end{remark}


\section{The implicit regularization of a ReLU neuron is approximately the $\ell_2$ norm}
\label{sec:positive result}

While the implicit regularization of ReLU neuron cannot be expressed as a function $\calr(\bw)$, in this section we show that it can be expressed approximately, within a factor of $2$, by the $\ell_2$ norm. This implies that even without early stopping, if the data can be labeled by a ReLU neuron with small $\ell_2$ norm, then gradient flow will converge to a ReLU neuron whose $\ell_2$ norm is not much larger. Since a ReLU neuron is just a linear function composed with a fixed nonlinearity, this can be used to derive good statistical generalization guarantees, via standard techniques (cf. \cite{shalev2014understanding}).

\begin{theorem}
\label{thm:positive result}
Consider gradient flow on the objective given by Eq.~\ref{eq:problem}, where $\sigma:\reals \rightarrow \reals$ is a monotonically non-decreasing activation function (e.g., ReLU).
Assume that $\bw(\infty)$ exists and $\cl(\bw(\infty))=0$. Let $\bw^* \in \argmin_\bw \norm{\bw - \bw(0)} \text{ s.t. } \cl(\bw)=0$. 
Then, $\norm{\bw(\infty) - \bw(0)} \leq 2 \cdot \norm{\bw^* - \bw(0)}$.
\end{theorem}
\begin{proof}
First, note that
\begin{align*}
\inner{\nabla \cl(\bw(t)), \bw(t) - \bw^*}
&= \left<\sum_{i=1}^n \left( \sigma(\bx_i^\top \bw(t))-\sigma(\bx_i^\top \bw^*) \right) \sigma'(\bx_i^\top \bw(t)) \bx_i, \bw(t) - \bw^*\right>
\\
&= \sum_{i=1}^n \left( \sigma(\bx_i^\top \bw(t))-\sigma(\bx_i^\top \bw^*) \right) \sigma'(\bx_i^\top \bw(t)) (\bx_i^\top \bw(t) - \bx_i^\top \bw^*)~.
\end{align*}
Since $\sigma$ is monotonically non-decreasing then 
\[
\left(  \sigma(\bx_i^\top \bw(t))-\sigma(\bx_i^\top \bw^*)  \right) (\bx_i^\top \bw(t) - \bx_i^\top \bw^*) \geq 0~,
\]
and $\sigma'(\bx_i^\top \bw(t)) \geq 0$.
Hence, 
\[
\inner{\nabla \cl(\bw(t)), \bw(t) - \bw^*} \geq 0~.
\]

Therefore, we have
\[
\frac{d}{\mathop{dt}} \left(\frac{1}{2} \norm{\bw(t) - \bw^*}^2 \right)
= \inner{\bw(t) - \bw^*, \dot{\bw}(t)}
= \inner{\bw(t) - \bw^*, -\nabla \cl(\bw(t))}
\leq 0~.
\]
Thus, $\norm{\bw(\infty) - \bw^*} \leq \norm{\bw(0) - \bw^*}$. 

Hence, 
\begin{align*}
\norm{\bw(\infty) - \bw(0)} 
&= \norm{\bw(\infty) - \bw^* + \bw^* - \bw(0)}
\\
&\leq \norm{\bw(\infty) - \bw^*} + \norm{\bw^* - \bw(0)}
\\
&= 2 \cdot \norm{\bw^* - \bw(0)}~.
\end{align*}
\end{proof}

\begin{remark}
	Note that Theorem~\ref{thm:positive result} implies that if $\bw(0) = \zero$ then $\norm{\bw(\infty)} \leq  2 \cdot \norm{\bw^*}$.
	By Theorem~\ref{thm:R trivial}, the implicit regularization cannot be expressed as a function of $\bw$. 
	Hence, while the implicit regularization of a ReLU neuron cannot be expressed exactly as a function of $\bw$, it can be expressed approximately, within a factor of $2$, by the $\ell_2$ norm.	
\end{remark}

\section{Depth-$2$ ReLU networks}
\label{sec:depth 2}

In Section~\ref{sec:not function of w}, we showed that the implicit regularization of a single ReLU neuron is not expressible by any nontrivial function of $\bw$. In this section we will show an analogous result for networks with one hidden ReLU neuron (namely, $\bx\mapsto v \cdot \sigma(\inner{\bx,\bw})$), which is the simplest case of a depth-$2$ ReLU network. Since our focus here is on impossibility results, then our results for this simple case imply impossibility results also for more general cases.

By \cite{du2018algorithmic}, in feed-forward ReLU networks, gradient flow enforces the differences between square norms across different layers to remain invariant. Thus, if gradient flow starts from a point close to $\zero$, then the magnitudes of all layers are automatically balanced. In the case of single-hidden-neuron networks, it implies that $\norm{\bw}$ (the norm of the weights vector of the first layer) is roughly equal to $|v|$ (the weight in the second layer). Hence, gradient flow induce a bias toward balanced layers. However, this bias is considered weak and does not allow us to derive generalization guarantees. 
Hence, there has been much effort to characterize the implicit regularization and to understand whether it is related to properties such as small norms, sparsity, or low ranks (cf. \cite{gunasekar2018implicit,razin2020implicit,arora2019implicit,woodworth2020kernel,belabbas2020implicit,li2020towards}).
We show that in single-hidden-neuron networks, the \emph{only} bias which can be specified by a regularization function $\calr(\btheta)$ (where $\btheta$ are the network parameters) is the balancedness property described above. Namely, $\calr(\btheta)$ is constant in the set of parameters that satisfy the balacedness property. 

Recall that in our study of single-neuron networks in Section~\ref{sec:not function of w}, we first analyzed the behavior of gradient flow for some specific inputs (Theorem~\ref{thm:not l2}), and then used this result in order to show that the implicit regularization function $\calr$ is trivial (Theorem~\ref{thm:R trivial}). Here, we proceed in a similar manner, but the first part is based on empirical results rather than on a theoretical proof. Thus, we first run gradient descent for some specific inputs, and observe empirically that it converges to points that clearly satisfy a certain technical property. Then, making the explicit assumption that the technical property holds, we prove that the regularization function $\calr(\btheta)$ is constant in the set of parameters that satisfy the balacedness property.     

We now proceed to the formal results.
Let $\btheta = (\bw,v) \in \reals^d \times \reals$, and consider the neural network $N_\btheta(\bx) = v \cdot \sigma(\inner{\bx,\bw})$, where $\sigma:\reals \rightarrow \reals$ is the ReLU function. Let $\{(\bx_i,y_i)\}_{i=1}^n$ be a training dataset, and let $X \in \reals^{n \times d}$ be the corresponding data matrix. We analyze the implicit regularization of gradient flow on the objective 
\begin{equation}
\label{eq:problem 2}
	\cl_{X,\by}(\btheta) := \frac{1}{2} \sum_{i=1}^n \left( N_\btheta(\bx_i) - y_i \right)^2~. 
\end{equation}
We assume that the data is realizable, i.e., $\min_\btheta \cl_{X,\by}(\btheta)=0$.
The run of gradient flow starts from $\btheta(0) = (\bw(0),v(0))$, where $\bw(0)=\zero$, and $v(0)$ is a small  number\footnote{Note that if we start from $(\zero,0)$ then gradient flow stays at this point indefinitely, since the gradient there is $\zero$. This is avoided by making either the scalar $v(0)$ or the vector $\bw(0)$ non-zero, and we focus on the former as it makes the analysis cleaner.}. 
Let $\btheta^\epsilon(t)$ be the trajectory of gradient flow where 
$\btheta^\epsilon(0)=(\zero,\epsilon)$ for some $\epsilon>0$,
and let $\btheta^\epsilon(\infty) = (\bw^\epsilon(\infty),v^\epsilon(\infty)):= \lim_{t \rightarrow \infty} \btheta^\epsilon(t)$ (assuming that the limit exists). In addition, we define the {\em limit point} of gradient flow (on inputs $(X,\by)$) as $\btheta^* = (\bw^*,v^*) := \lim_{\epsilon \rightarrow 0^+} \btheta^\epsilon(\infty)$, assuming the limit exists.
This follows a standard practice in analyzing the implicit regularization of gradient flow, where we consider the flow's limit assuming it starts infinitesimally close to $\zero$ (see, e.g.,  \cite{gunasekar2018implicit,razin2020implicit,arora2019implicit,woodworth2020kernel,li2020towards}).

We are interested in understanding the properties of any possible implicit regularization function $\calr:\reals^d \times \reals \rightarrow \reals$: Namely, a function that for every input $(X,\by) \in \reals^{n \times d} \times \reals^n$ to gradient flow, if $\btheta^*$ exists and $\cl_{X,\by}(\btheta^*)=0$, then we have $\btheta^* \in \argmin_\btheta \calr(\btheta) \text{ s.t. } \cl_{X,\by}(\btheta)=0$. Whereas for a single ReLU neuron, we showed that $\calr$ is necessarily constant, the situation for a single hidden-neuron is more complex, because gradient flow is already known to induce the following nontrivial balacedness property\footnote{\cite{du2018algorithmic} showed the lemma for the more general case of fully-connected networks with homogeneous activation functions. See Appendix~\ref{app:proof of lemma v_t} for a simpler proof for our setting.}:

\begin{lemma}[\cite{du2018algorithmic}]
\label{lemma:v_t}
	Let $\btheta(t) = (\bw(t),v(t))$ be the trajectory of gradient flow on Eq.~\ref{eq:problem 2}, starting from some $\btheta(0) \in \reals^d \times \reals$. Then 
	\[
	\frac{\mathop{d}}{\mathop{d t}} \left( \norm{\bw}^2 - v^2 \right) = 0~.
	\]
\end{lemma}

As a consequence, we have the following corollary.
\begin{corollary}
\label{cor:v(t)}
	Let $\epsilon>0$ and let $\btheta^\epsilon(t) = (\bw(t),v(t))$.
	For every $t \geq 0$ we have $(v(t))^2 = \norm{\bw(t)}^2 + \epsilon^2 > 0$. 	
\end{corollary}

By Corollary~\ref{cor:v(t)}, we have $v^\epsilon(\infty) = \sqrt{ \norm{\bw^\epsilon(\infty)}^2 + \epsilon^2}$, and hence $v^* = \norm{\bw^*}$. As a result, not all parameters $(\bw^*,v^*)$ can be reached by gradient flow -- only those which satisfy the ``balancedness'' property above. However, this property in itself is rather weak, and does not necessarily induce bias towards properties which may aid in generalization, such as small norms or sparsity. Is there any stronger implicit regularization at play here? To answer that, it is natural to consider the behavior of the implicit regularization function $\calr$ on the set $\cs := \{(\bw,v) \in \reals^{d} \times \reals: \norm{\bw}=v\}$ of ``balanced'' parameters. In what follows, we will argue that $\calr$ is constant on $\cs \setminus \{\zero\}$. Thus, $\calr$ may induce a bias toward $\cs$, namely, enforce the ``balancedness" of $\bw^*$ and $v^*$, but it does not induce any additional bias within $\cs$.


In order to show that $\calr$ is constant in $\cs \setminus \{\zero\}$, we start with an empirical study of gradient descent for some specific inputs. Then, under a mild assumption based on the empirical results, we prove that $\calr$ is constant.

\subsection{Empirical results}

In Theorem~\ref{thm:not l2} we showed that the implicit regularization of a single ReLU neuron is not the $\ell_2$ norm. We now demonstrate an analogous behavior in single-hidden-neuron networks. 

Let $X_\gamma \in \reals^{3 \times 3}$ be the matrix from Theorem~\ref{thm:not l2}. Let $\by \in \reals^3$ be the vector from Theorem~\ref{thm:not l2} with $\alpha=1$, i.e., $\by = (16,18,0)^\top$. 
Let $\cq = \{(5,-1,q)^\top: q \leq 0\}$. 
Note that for every $\bu \in \cq$ and $\gamma \geq 0$ we have $\sigma(X_\gamma \bu)=\by$.
By Theorem~\ref{thm:not l2}, for different values of $\gamma \geq 0$, gradient flow on a single neuron with the input $(X_\gamma,\by)$ converges to different points $\bw^\gamma$ in a set of global minima. 
We now show empirically that a similar behavior appears also in single-hidden-neuron networks: 
For different values of $\gamma \geq 0$, gradient descent with the input $(X_\gamma,\by)$, starting from $(\zero,\epsilon)$ for a fixed $\epsilon>0$, converges to different points $\btheta^\gamma=(\bw^\gamma,v^\gamma)$ with $v^\gamma \bw^\gamma \in \cq$. We denote $\bu^\gamma = v^\gamma \bw^\gamma$.

In Figure~\ref{fig:sim2}, we show the trajectories $\bu(t) = v(t) \cdot \bw(t)$ of gradient descent starting from $\btheta(0)=(\zero,0.001)$, for $\gamma \in \{0,1,2,5,20\}$. The grey line denotes the set $\cq$. 
We note that while the trajectories in Figure~\ref{fig:sim2} have similar shapes to the trajectories shown in Figure~\ref{fig:sim} for the case of a single-neuron network, they are not identical. Indeed, the dynamics of the problems are different. This experiment demonstrates that for $\epsilon=0.001$ and different values of $\gamma \geq 0$, gradient descent converges to different points $\bu^\gamma \in \cq$.

Since we are interested in gradient flow starting from $(\zero,\epsilon)$ for $\epsilon \rightarrow 0$, we need to consider $\bu^\gamma$ where $\epsilon \rightarrow 0$. 
Consider $\gamma \in \{0,5\}$, and let $\bu^0=(u_1^0,u_2^0,u_3^0)$ and $\bu^5=(u_1^5,u_2^5,u_3^5)$.
We now show empirical evidence that strongly supports the following assumption, which will be the key used to prove that the implicit regularization function is constant as discussed earlier:
\begin{assumption}
	\label{ass:1}
	For every sufficiently small $\epsilon>0$, gradient flow starting from $(\zero,\epsilon)$ on the inputs $(X_0,\by)$ and $(X_5,\by)$, converges to zero loss. Moreover, $\lim_{\epsilon \rightarrow 0^+}u_3^5$ exists and is negative. 
\end{assumption}
\stam{
\begin{assumption}
\label{ass:1}
	For the inputs $(X_0,\by)$ and $(X_5,\by)$, the limit points of gradient flow are $\btheta^*=(\bw^*,v^*)$ and $\tilde{\btheta}^*=(\tilde{\bw}^*,\tilde{v}^*)$ (respectively), such that $v^* \bw^* = (5,-1,0)^\top$ and $\tilde{v}^* \tilde{\bw}^* = (5,-1,s)^\top$ for some $s<0$.
\end{assumption}
}

We turn to present this empirical evidence. 
Since we are interested in simulating gradient flow, we used gradient descent with a small learning rate of $10^{-5}$.
We ran gradient descent starting from $(\zero,\epsilon)$, for $\gamma \in \{0,5\}$ and 
$\epsilon \in \{10^{-i}: i=0,\ldots,5\}$. 
In all cases we reached loss smaller than $10^{-15}$, which suggests that the first part of the assumption holds. 
We now turn to the second part of the assumption. 
In Figure~\ref{fig:sim-epsilon} we show for different values of $\epsilon$, the value of $u_3^5$ in the last iteration of gradient descent starting from $(\zero,\epsilon)$ and reaching loss smaller than $10^{-15}$. 
This experiment suggests that $\lim_{\epsilon \rightarrow 0^+}u_3^5$ exists and is negative.


\begin{figure}[t]
	\begin{subfigure}{0.485\textwidth}
		\includegraphics[width=\textwidth]{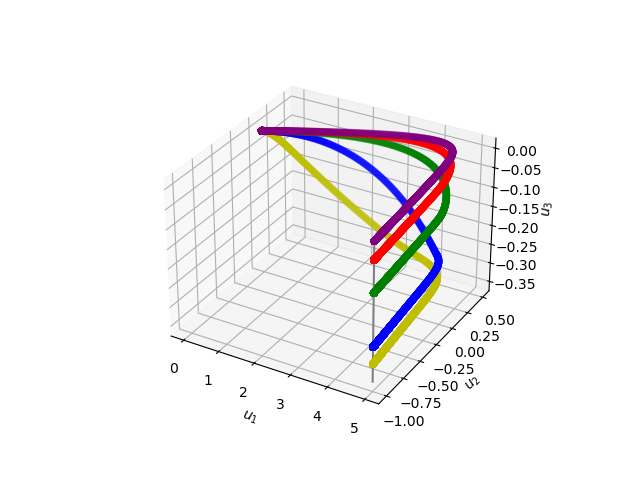} 
		\caption{The trajectories $\bu(t)$ for the input $(X_\gamma,\by)$ starting from $\btheta(0)=(\zero,0.001)$, for $\gamma=0$ (purple), $\gamma=1$ (red), $\gamma=2$ (green), $\gamma=5$ (blue), and $\gamma=20$ (yellow).}
		\label{fig:sim2}
	\end{subfigure}
	\quad
	\begin{subfigure}{0.485\textwidth}
		\includegraphics[width=\textwidth]{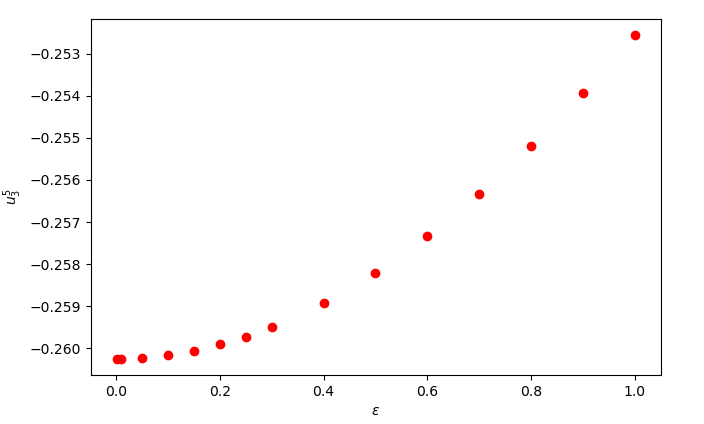}
		\caption{The value of $u_3^5$ in the last iteration of gradient descent starting from $(\zero,\epsilon)$ and reaching loss smaller than $10^{-15}$.}
		\label{fig:sim-epsilon}
	\end{subfigure}
	\caption{Experiments on single-hidden-neuron networks.}
\end{figure}

Finally, we show that Assumption~\ref{ass:1} easily implies the following corollary:
\begin{corollary}
\label{cor:from ass1}
Under Assumption~\ref{ass:1}, for the inputs $(X_0,\by)$ and $(X_5,\by)$, the limit points of gradient flow are $\btheta^*=(\bw^*,v^*)$ and $\tilde{\btheta}^*=(\tilde{\bw}^*,\tilde{v}^*)$ (respectively), such that $v^* \bw^* = (5,-1,0)^\top$ and $\tilde{v}^* \tilde{\bw}^* = (5,-1,s)^\top$ for some $s<0$.
\end{corollary} 
\begin{proof}
First, since by Assumption~\ref{ass:1} for every sufficiently small $\epsilon>0$ the loss converges to zero, then for every sufficiently small $\epsilon>0$ we have $(u_1^0,u_2^0)=(u_1^5,u_2^5)=(5,-1)$, and hence we also have $\lim_{\epsilon \rightarrow 0^+}(u_1^0,u_2^0)=\lim_{\epsilon \rightarrow 0^+}(u_1^5,u_2^5)=(5,-1)$.
Clearly, for every $\epsilon > 0$ we have $u_3^0 = 0$, since if $\gamma=0$ then the third component of $\bu(t)$ is $0$ for every $t \geq 0$. Therefore, $\lim_{\epsilon \rightarrow 0^+}u_3^0 = 0$.
Finally, by Assumption~\ref{ass:1}, $\lim_{\epsilon \rightarrow 0^+}u_3^5$ exists and is negative. Overall, we have $\lim_{\epsilon \rightarrow 0^+}\bu^5 = (5,-1,s)^\top$ for some $s < 0$, and $\lim_{\epsilon \rightarrow 0^+}\bu^0 = (5,-1,0)^\top$. 
\end{proof}

\subsection{$\calr$ is constant in the set of balanced parameters}

We now show that 
the regularization function $\calr$ is constant in the set $\cs$ of balanced parameters.
Since the limit point $\btheta^*=(\bw^*, v^*) \in \reals^d \times \reals$ of gradient flow is in $\cs$, then it is uniquely determined by the point $v^*\bw^* \in \reals^d$, and hence, intuitively, we can think about $\btheta^*$ as a point in $\reals^d$. 
More formally, let $\Psi: \reals^d \rightarrow \cs$ be such that $\Psi(\bu) = \left(\frac{\bu}{\sqrt{\norm{\bu}}}, \sqrt{\norm{\bu}} \right)$ for $\bu \neq \zero$, and $\Psi(\zero) = (\zero,0)$. 
Note that for every $(\bw,v) \in \cs$ we have $\Psi(v \bw) = (\bw,v)$, and that $\Psi$ is bijective.
Let $\calr': \reals^d \rightarrow \reals$ be such that $\calr'(\bu) = \calr(\Psi(\bu))$. 

We state our main theorem on single-hidden-neuron networks:

\begin{theorem}
\label{thm:R' constant}
	Under Assumption~\ref{ass:1}, the function 
	$\calr':\reals^3 \rightarrow \reals$ is constant in $\reals^3 \setminus \{\zero\}$. Therefore, 
	$\calr:\reals^3 \times \reals \rightarrow \reals$ is constant in $\cs \setminus \{\zero\}$.
\end{theorem}

Intuitively, the theorem follows from the following argument. By Corollary~\ref{cor:from ass1}, for the input $(X_0,\by)$, gradient flow converges to $\btheta^*$ such that $\bu^*:=\Psi^{-1}(\btheta^*)=(5,-1,0)^\top$, even though every $\bu \in \cq =  \{(5,-1,q)^\top: q \leq 0\}$ corresponds to a zero-loss solution. Hence, gradient flow prefers $\bu^*$ over all other points in $\cq$. Similarly, for the input $(X_5,\by)$ gradient flow converges to $\tilde{\btheta}^*$ such that $\tilde{\bu}^*:=\Psi^{-1}(\tilde{\btheta}^*)=(5,-1,s)^\top$ for some $s<0$, even though every $\bu \in \cq$ corresponds to a zero-loss solution. Hence, it prefers $\tilde{\bu}^*$ over all other points in $\cq$. It implies that $\calr'(\bu^*) = \calr'(\tilde{\bu}^*) = \min_{\bu \in \cq} \calr'(\bu)$. Then, by rotating and scaling the inputs to gradient flow, and by using Lemma~\ref{lemma:R constant}, we show that $\calr'$ is constant in $\reals^3 \setminus \{\zero\}$. 

To prove the theorem more formally, we start with some additional notations. For an input $(X,\by) \in \reals^{n \times d} \times \reals^n$ to gradient flow, we denote 
$\cu_{X,\by} = \{\bu \in \reals^d: \cl_{X,\by}(\Psi(\bu))=0\}$, namely, the vectors in $\reals^d$ that correspond to global minima of $ \cl_{X,\by}$.
Let $\bu \in \reals^d$ and let $(\bw,v)=\Psi(\bu)$. Note that $\bu \in \cu_{X,\by}$ iff $\by = v \cdot \sigma(X\bw)$ iff $\by = \sigma(v X \bw) = \sigma(X \bu)$.
Therefore, we have $\cu_{X,\by} = \{\bu \in \reals^d: \sigma(X \bu)=\by\}$.
Let $\bu^* = v^* \bw^* = \Psi^{-1}(\btheta^*)$, namely, the vector that corresponds to the limit point. 
Since $\btheta^* \in \argmin_{\{\btheta \in \cs: \cl_{X,\by}(\btheta)=0\}} \calr(\btheta)$, then $\bu^* \in \argmin_{\bu \in  \cu_{X,\by}} \calr'(\bu)$. Indeed, we have 
\[
\calr'(\bu^*) 
= \calr(\Psi(\bu^*)) 
= \calr(\btheta^*) 
= \min_{\{\btheta \in \cs: \cl_{X,\by}(\btheta)=0\}}\calr(\btheta) 
= \min_{\{\bu \in \reals^d: \cl_{X,\by}(\Psi(\bu))=0\}}\calr(\Psi(\bu))
= \min_{\bu \in \cu_{X,\by}}\calr'(\bu)~,
\]
and $\bu^* \in \cu_{X,\by}$ since  $\cl_{X,\by}(\Psi(\bu^*))=\cl_{X,\by}(\btheta^*)=0$.

Note that by Corollary~\ref{cor:from ass1},
for the inputs $(X_0,\by)$ and $(X_5,\by)$, the limit points of gradient flow are $\btheta^*=(\bw^*,v^*)$ and $\tilde{\btheta}^*=(\tilde{\bw}^*,\tilde{v}^*)$ (respectively), such that $\bu^* = \Psi^{-1}(\btheta^*) = v^* \bw^* = (5,-1,0)^\top$ and $\tilde{\bu}^* = \Psi^{-1}(\tilde{\btheta}^*) = \tilde{v}^* \tilde{\bw}^* = (5,-1,s)^\top$, and therefore we have the following.
Let $\bu' = \tilde{\bu}^* - \bu^* = (0,0,s)^\top$. We have 
$\bu^*,\tilde{\bu}^*,\bu' \neq \zero$, and $\inner{\bu^*,\bu'}=0$. 
Moreover, since for every $\bu \in \cq =  \{(5,-1,q)^\top: q \leq 0\}$ we have $\sigma(X_0 \bu)=\sigma(X_5 \bu) = \by$, then $\cq \subseteq \cu_{X_0,\by} \cap \cu_{X_5,\by}$, 
and therefore for every $\beta \geq 0$ we have $\bu^* + \beta \bu' \in \cu_{X_0,\by} \cap \cu_{X_5,\by}$.
Since $\bu^* \in \argmin_{\bu \in  \cu_{X_0,\by}} \calr'(\bu)$ and $\tilde{\bu}^* \in \argmin_{\bu \in  \cu_{X_5,\by}} \calr'(\bu)$, then we have $\calr'(\bu^*) = \calr'(\tilde{\bu}^*) = \min_{\beta \geq 0} \calr'(\bu^* + \beta \bu')$.

We now utilize the
following lemma,
which 
considers the effect on the limit point of rotating and scaling the input to gradient flow. The lemma follows by analyzing the trajectories of gradient flow with the modified inputs. See Appendix~\ref{app:proof of lemma rotation and scaling} for a proof.

\begin{lemma}
\label{lemma:rotation and scaling}
	Let $(X,\by) \in \reals^{n \times d} \times \reals^n$ be an input to gradient flow, let $\btheta^* = (\bw^*,v^*)$ be its limit point, and let $\bu^* = v^* \bw^* = \Psi^{-1}(\btheta^*)$.
	\begin{enumerate}
		\item Let $M \in SO(d)$ be a rotation matrix. The limit point of gradient flow with input $(XM^\top,\by)$ is $\btheta^*_M = (M \bw^*,v^*)$. Thus, $\Psi^{-1}(\btheta^*_M) = M \bu^*$. 
		
		\item Let $\alpha > 0$. The limit point of gradient flow with input $(X, \alpha \by)$ is $\btheta^*_\alpha = \sqrt{\alpha} \cdot (\bw^*,v^*)$. Thus, $\Psi^{-1}(\btheta^*_\alpha) = \alpha \bu^*$.
	\end{enumerate}
\end{lemma}

By rotating and scaling the inputs from Corollary~\ref{cor:from ass1}, and using Lemma~\ref{lemma:rotation and scaling}, we obtain the following corollary (see Appendix~\ref{app:proof of cor R' properties} for a formal proof).

\begin{corollary}
\label{cor:R' properties}
	Under Assumption~\ref{ass:1}, there is a vector $\bu^* \in \reals^3 \setminus \{\zero\}$ and a constant $c>0$, such that for every $\alpha > 0$ and $\bu \in \reals^3$ with $\norm{\bu} = \alpha \cdot \norm{\bu^*}$, and for every $\bu_\bot \in \reals^3$ such that $\inner{\bu_\bot,\bu}=0$ and $\norm{\bu_\bot} = \alpha c$, we have $\calr'(\bu) = \calr'(\bu + \bu_\bot) = \min_{\beta \geq 0}\calr'(\bu + \beta \bu_\bot)$.  
\end{corollary}

Finally, Theorem~\ref{thm:R' constant} follows immediately from Corollary~\ref{cor:R' properties} and Lemma~\ref{lemma:R constant}.

\begin{remark}[$\calr$ is constant in $\cs \setminus \{\zero\}$ also for $d>3$] 
	Theorem~\ref{thm:R' constant} implies that $\calr':\reals^d \rightarrow \reals$ is trivial for $d=3$. By padding the inputs from Assumption~\ref{ass:1} with zeros, we can obtain a $d$-dimensional version of Assumption~\ref{ass:1} for $d>3$. Then, using the same arguments as above, it implies that Theorem~\ref{thm:R' constant} holds for every $d \geq 3$.
\end{remark}

\stam{
\begin{proposition}
	Let $X, \tilde{X}, \by$ be the matrices and vector from Assumption~\ref{ass:1}, and let $\bu^*, \tilde{\bu}^*,\bu'$ be the corresponding vectors.	
	Under Assumption~\ref{ass:1}, we have:
	\begin{enumerate}
		\item Let $M \in SO(3)$ be a rotation matrix. By replacing the matrices $X$ and $\tilde{X}$ with $XM^\top$ and $\tilde{X} M^\top$ (respectively), the corresponding vectors $\bu^*$ and $\tilde{\bu}^*$ are changed to $M \bu^*$ and $M \tilde{\bu}^*$. Moreover,  for both inputs $XM^\top, \tilde{X} M^\top$ and every $\alpha \geq 0$, we have $\cl(\Psi(M\bu^* + \alpha M\bu')) = 0$. 
			
		\item Let $\beta > 0$ be a constant. By replacing $\by$ with $\beta \by$, 
	\end{enumerate}
\end{proposition}
}

\subsection*{Acknowledgements}
This research is supported in part by European Research Council (ERC) grant 754705.

\bibliographystyle{abbrvnat}
\bibliography{bib}

\appendix

\section{Proof of Theorem~\ref{thm:not l2}}
\label{app:not function of w}

\begin{lemma}
\label{lemma:solving the diffenrential equation}
Let $A \in \reals^{d \times d}$ be an invertible matrix, and let $\by \in \reals^d$. Consider the dynamics of $\bw(t)$ given by the differential equation
\begin{equation*}
\label{eq:dif w}
\dot{\bw}(t) = - A^\top (A \bw(t) -\by)~.
\end{equation*}
Then, we have
\begin{equation}
\label{eq:sol w}
\bw(t) = \exp(-t A^\top A) \left( \bw(0) - A^{-1} \by \right) + A^{-1} \by~,
\end{equation}
where $\exp(\cdot)$ denotes the matrix exponential.
\end{lemma}
\stam{
\begin{proof}
By differentiating the expression for $\bw(t)$ in Eq.~\ref{eq:sol w} it can be verified that Eq.~\ref{eq:dif w} holds. The solution is unique by the Cauchy-Lipschitz theorem \citep{coddington1955theory}.
\end{proof}
}
The proof of Lemma~\ref{lemma:solving the diffenrential equation} follows easily by differentiating the expression in Eq.~\ref{eq:sol w}.

\subsection{Proof for $\gamma=0$}

The trajectory $\bw(t)$ obeys the dynamics
\begin{align*}
\dot{\bw}(t) 
&= - \nabla \cl(\bw(t)) 
= - \sum_{i=1}^3 \left( \sigma(\bx_i^\top \bw(t)) - y_i \right) \sigma'(\bx_i^\top \bw(t)) \bx_i
\\
&= - \sum_{i=1}^2 \left( \sigma(\bx_i^\top \bw(t)) - y_i \right) \sigma'(\bx_i^\top \bw(t)) \bx_i~,
\end{align*}
where the last equality is since $\bx_3 = \zero$.
Hence, since we also have $\bw(0) = \zero$, then $\bw(t)$ is spanned by $\bx_1,\bx_2$. Note that the third components in $\bx_1,\bx_2$ are $0$, and therefore the third component in $\bw(t)$ is $0$ for all $t \geq 0$.
For a vector $\bv \in \{\bw(t),\bx_1,\bx_2,\by\}$, we denote $\tilde{\bv} = (v_1,v_2) \in \reals^2$. We also denote by $\tilde{X} \in \reals^{2 \times 2}$ the matrix whose rows are $\tilde{\bx}_1^\top$ and $\tilde{\bx}_2^\top$.
Note that for $i \in \{1,2\}$ we have $\bx_i^\top \bw(t) = \tilde{\bx}_i^\top \tilde{\bw}(t)$. 
Thus, we have 
\begin{equation}
\label{eq:gamma0 dif}
\dot{\tilde{\bw}}(t)
= - \sum_{i=1}^2 \left( \sigma(\tilde{\bx}_i^\top \tilde{\bw}(t)) - y_i \right) \sigma'(\tilde{\bx}_i^\top \tilde{\bw}(t)) \tilde{\bx}_i~.  
\end{equation}
If $\tilde{\bx}_i^\top \tilde{\bw}(t) \geq 0$  for $i \in \{1,2\}$ then the above equals 
\begin{equation*}
\label{eq:gamma0 dif B}
- \sum_{i=1}^2 (\tilde{\bx}_i^\top \tilde{\bw}(t) - y_i ) \tilde{\bx}_i
= - \tilde{X}^\top (\tilde{X}\tilde{\bw}(t) - \tilde{\by})~.  
\end{equation*}
By Lemma~\ref{lemma:solving the diffenrential equation}, the solution to the above equation is 
\begin{equation}
\label{eq:gamma0 traj}
\tilde{\bw}(t) 
= \exp\left(-t \tilde{X}^\top \tilde{X}\right)( \tilde{\bw}(0) - \tilde{X}^{-1} \tilde{\by} ) + \tilde{X}^{-1} \tilde{\by}~.
\end{equation}

Let $\tilde{\cb} = \{\tilde{\bw} \in \reals^2: \tilde{\bx}_1^\top \tilde{\bw} \geq 0, \tilde{\bx}_2^\top \tilde{\bw} \geq 0\}$. 
We will show that the trajectory from Eq.~\ref{eq:gamma0 traj} satisfies $\tilde{\bw}(t) \in \tilde{\cb}$ for every $t \geq 0$. Therefore, it satisfies Eq.~\ref{eq:gamma0 dif} for every $t \geq 0$. It implies that $\tilde{\bw}(\infty) =   \tilde{X}^{-1} \tilde{\by} = \alpha \cdot (5,-1)^\top$, and therefore $\bw(\infty) = \alpha \cdot (5,-1,0)^\top$ as required.

\begin{lemma}
	The trajectory from Eq.~\ref{eq:gamma0 traj} satisfies $\tilde{\bw}(t) \in \tilde{\cb}$ for every $t \geq 0$.
\end{lemma}
\begin{proof}
	The eigenvalues of $\tilde{X}^\top \tilde{X}$ are $\tilde{\lambda}_1 = 15+5\sqrt{5}$ and $\tilde{\lambda}_2 = 15-5\sqrt{5}$. Let $\tilde{D}=\diag(\tilde{\lambda}_1,\tilde{\lambda}_2)$, and let $\tilde{U} \in \reals^{2 \times 2}$ such that $\tilde{X}^\top \tilde{X} = \tilde{U} \tilde{D} \tilde{U}^\top$.
	By Eq.~\ref{eq:gamma0 traj} and since $\tilde{\bw}(0)=\zero$, we have
	\begin{equation*}
		\label{eq:gamma0 traj b}
		\tilde{\bw}(t) 
		= \tilde{U} \exp\left(-t \tilde{D}\right) \tilde{U}^\top \left(\zero - \alpha(5,-1)^\top \right) + \alpha(5,-1)^\top~,
	\end{equation*}
	where $\exp\left(-t \tilde{D}\right) = \diag(e^{-\tilde{\lambda}_1 t},e^{-\tilde{\lambda}_2 t})$.
	
	By straightforward calculations we obtain
	\begin{align*}
		\frac{1}{\alpha} \cdot \tilde{\bx}_1^\top \tilde{\bw}(t)
		&= - (3,-1) \tilde{U} \exp\left(-t \tilde{D}\right) \tilde{U}^\top (5,-1)^\top + (3,-1) (5,-1)^\top
		\\
		&= - \left(8 + 2 \sqrt{5} \right) e^{-\tilde{\lambda}_1 t} - \left(8 - 2 \sqrt{5} \right) e^{-\tilde{\lambda}_2 t} + 16
		\\
		& \geq - \left(8 + 2 \sqrt{5} \right) - \left(8 - 2 \sqrt{5} \right) + 16
		= 0~. 
	\end{align*}

	Moreover, we have
	\begin{align*}
		\frac{1}{\alpha} \cdot \tilde{\bx}_2^\top \tilde{\bw}(t)
		&= - (4,2) \tilde{U} \exp\left(-t \tilde{D}\right) \tilde{U}^\top (5,-1)^\top + (4,2) (5,-1)^\top
		\\
		&= -\left(5\sqrt{5} + 9 \right) e^{-\tilde{\lambda}_1 t} + \left(5\sqrt{5} - 9\right) e^{-\tilde{\lambda}_2 t} + 18~.
	\end{align*}
	Since $\tilde{\lambda}_1 \geq \tilde{\lambda}_2 \geq 0$, then the above is at least
	\[
		-\left(5\sqrt{5} + 9 \right) e^{-\tilde{\lambda}_1 t} + \left(5\sqrt{5} - 9\right) e^{-\tilde{\lambda}_1 t} + 18
		= -18 e^{-\tilde{\lambda}_1 t} + 18
		\geq 0~.
	\]
\end{proof}

\subsection{Proof for $\gamma > 0$}

We show that the trajectory $\bw(t)$ consists of two parts. In the first part, where $t \in [0,t_1]$ for some $t_1 > 0$, we have $\bx_i^\top \bw(t) \geq 0$  for every $1 \leq i \leq 3$. Then, for $t \in [t_1,\infty)$, we have $\bx_1^\top \bw(t) \geq 0$,  $\bx_2^\top \bw(t) \geq 0$, and $\bx_3^\top \bw(t) \leq 0$. 

\subsubsection{Part 1}

Recall that the trajectory $\bw(t)$ obeys the dynamics
\begin{equation}
\label{eq:gamma positive dif}
\dot{\bw}(t) 
= - \nabla \cl(\bw(t)) 
= - \sum_{i=1}^3 \left( \sigma(\bx_i^\top \bw(t)) - y_i \right) \sigma'(\bx_i^\top \bw(t)) \bx_i~.
\end{equation}
If $\bx_i^\top \bw(t) \geq 0$  for every $1 \leq i \leq 3$ then the above equals 
\[
- \sum_{i=1}^3 (\bx_i^\top \bw(t) - y_i ) \bx_i
= - X^\top (X  \bw(t) - \by)~.
\]

By Lemma~\ref{lemma:solving the diffenrential equation}, the solution to the above equation is
\begin{equation}
\label{eq:traj1}
\bw(t) = \exp\left(-t X^\top X\right) \left( \bw(0) - X^{-1} \by \right) + X^{-1} \by~.
\end{equation}

Let $\cb = \{\bw \in \reals^3: \bx_i^\top \bw \geq 0, \; i = 1 \ldots 3 \}$. We will show that the trajectory from Eq.~\ref{eq:traj1} satisfies $\bw(t) \in \cb$ for every $t \in [0,t_1]$ for some $t_1>0$. Therefore, it satisfies Eq.~\ref{eq:gamma positive dif} for $t \in [0,t_1]$. Hence, gradient flow follows the trajectory from Eq.~\ref{eq:traj1} where $t \in [0,t_1]$.
We will also investigate the values of the components of $\bw(t_1)$, since it is the starting point of the second part of the trajectory.

Note that for every $\gamma$ we have $X \alpha (5,-1,1)^\top = \alpha (16,18,0)^\top = \by$, and therefore
$X^{-1}  \by = \alpha (5,-1,1)^\top$.
\stam{
\begin{equation}
\label{eq:xinv y}
X^{-1}  \by = (5,-1,1)^\top~.
\end{equation}
}
Since we also have $\bw(0)=\zero$, then Eq.~\ref{eq:traj1} implies
\begin{equation*}
\bw(t) = \alpha (5,-1,1)^\top - \exp\left(-t X^\top X\right) \alpha (5,-1,1)^\top~.
\end{equation*}
Let $\lambda_1,\lambda_2,\lambda_3$ be the eigenvalues of $X^\top X$, let $D=\diag(\lambda_1,\lambda_2,\lambda_3)$, and let $U \in \reals^{3 \times 3}$ such that $X^\top X = U D U^\top$. 
Thus, we have
\begin{equation}
\label{eq:traj1b}
\bw(t) = \alpha (5,-1,1)^\top - U \exp(-tD) U^\top \alpha (5,-1,1)^\top~,
\end{equation}
where $\exp(-tD) = \diag\left(e^{-\lambda_1 t},e^{-\lambda_2 t},e^{-\lambda_3 t}\right)$.
 
\begin{lemma}
\label{lemma:in B all gamma}
 	Let $\bw(t)$ be the trajectory from Eq.~\ref{eq:traj1b}. 
 	For every $\gamma \in \{1,2,5\}$, there exists $t_1>0$, such that $\bw(t) \in \cb$ for every $t \in [0,t_1]$, $\bx_3^\top \bw(t_1) = 0$, and $\bw(t_1) = \alpha \cdot  (a,b,-b)$ for some $a,b$ (independent of $\alpha$), where $a \in (4,5)$ and $b$ satisfies the following:
 	\begin{itemize}
 		\item If $\gamma=1$ then $b \in (0.035,0.045)$.
 		\item If $\gamma=2$ then $b \in (0.1,0.12)$.
 		\item If $\gamma=5$ then $b \in (0.2,0.22)$.
 	\end{itemize}
\end{lemma}
 
We defer the proof of Lemma~\ref{lemma:in B all gamma} to later sections (the proof for $\gamma=1$ is given in Section~\ref{sec:in B gamma 1}, for $\gamma=2$ in Section~\ref{sec:in B gamma 2}, and for $\gamma=5$ in Section~\ref{sec:in B gamma 5}). 
 
\subsubsection{Part 2}
 
In Lemma~\ref{lemma:in B all gamma} we established the trajectory for $t \in [0,t_1]$.
We need to find the trajectory that satisfies the dynamics in Eq.~\ref{eq:gamma positive dif} for $t \in [t_1,\infty)$, with the initial value $\bw(t_1)$ obtained in Lemma~\ref{lemma:in B all gamma} . 

Let $\cb' = \{\bw \in \reals^3: \bx_1^\top \bw \geq 0, \bx_2^\top \bw \geq 0, \bx_3^\top \bw \leq 0\}$.
Consider the dynamics from Eq.~\ref{eq:gamma positive dif}, where $\bw(t) \in \cb'$.
Since  $\sigma(\bx_3^\top \bw(t)) = 0$ and $y_3=0$, and since $\bx_i^\top \bw(t) \geq 0$ for $i \in \{1,2\}$, then we have 
\begin{align*}
\dot{\bw}(t) 
&= - \sum_{i=1}^3 \left( \sigma(\bx_i^\top \bw(t)) - y_i \right) \sigma'(\bx_i^\top \bw(t)) \bx_i \nonumber
\\
&= - \sum_{i=1}^2 \left( \sigma(\bx_i^\top \bw(t)) - y_i \right) \sigma'(\bx_i^\top \bw(t)) \bx_i \nonumber
\\
&= - \sum_{i=1}^2 ( \bx_i^\top \bw(t) - y_i ) \bx_i~.
\end{align*}
We will find the trajectory $\bw(t)$ that satisfies
\begin{equation}
	\label{eq:traj2 dif}
	\dot{\bw}(t) = - \sum_{i=1}^2 ( \bx_i^\top \bw(t) - y_i ) \bx_i~,
\end{equation}
with the initial value $\bw(t_1)$, and show that $\bw(t) \in \cb'$ for every $t \in [t_1,\infty)$. Such a trajectory obeys the dynamics in Eq.~\ref{eq:gamma positive dif} for $t \in [t_1,\infty)$, with the initial value $\bw(t_1)$, as required.

\stam{
Since both $\bx_1$ and $\bx_2$ have $0$ as their third component, then $w_3(t) = w_3(t_1)$ for all $t \geq t_1$.
For a vector $\bv \in \{\bw(t),\bx_1,\bx_2,\by\}$, we denote $\tilde{\bv} = (v_1,v_2) \in \reals^2$. We also denote by $\tilde{X} \in \reals^{2 \times 2}$ the matrix whose rows are $\tilde{\bx}_1^\top$ and $\tilde{\bx}_2^\top$.
Note that for $i \in \{1,2\}$ we have $\tilde{\bx}_i^\top \tilde{\bw}(t) = \bx_i^\top \bw(t)$.
Thus, we have 
}

For a vector $\bv \in \{\bw(t),\bx_1,\bx_2,\by\}$, we denote $\tilde{\bv} = (v_1,v_2) \in \reals^2$. We also denote by $\tilde{X} \in \reals^{2 \times 2}$ the matrix whose rows are $\tilde{\bx}_1^\top$ and $\tilde{\bx}_2^\top$.
Since both $\bx_1$ and $\bx_2$ have $0$ as their third components, then $\bw(t)$ satisfies Eq.~\ref{eq:traj2 dif} iff the third component $w_3(t)$ satisfies 
\begin{equation}
\label{eq:w_3}
	\dot{w}_3(t)=0~,
\end{equation}
and the first two components $\tilde{\bw}(t)$ satisfy 
\begin{equation}
\label{eq:dif tilde}
\dot{\tilde{\bw}}(t) 
= - \sum_{i=1}^2 ( \tilde{\bx}_i^\top \tilde{\bw}(t) - y_i ) \tilde{\bx}_i
= - \tilde{X}^\top (\tilde{X} \tilde{\bw}(t) - \tilde{\by} ) ~.
\end{equation}

By Lemma~\ref{lemma:solving the diffenrential equation}, the trajectory that follows from Eq.~\ref{eq:dif tilde} is 
\begin{align}
\label{eq:traj2}
\tilde{\bw}(t) 
&= \exp\left(-(t-t_1) \tilde{X}^\top \tilde{X}\right) \left( \tilde{\bw}(t_1) - \tilde{X}^{-1} \tilde{\by} \right) + \tilde{X}^{-1} \tilde{\by} \nonumber
\\
&= \exp\left(-(t-t_1) \tilde{X}^\top \tilde{X}\right) \left( \tilde{\bw}(t_1) - \alpha (5,-1)^\top \right) + \alpha (5,-1)^\top~.
\end{align}

Consider the trajectory $\bw(t)$ for $t \geq t_1$, such that the first two components follow the trajectory $\tilde{\bw}(t)$ from the above equation, and the third component is constant, i.e., $w_3(t) = w_3(t_1)$ for all $t \geq t_1$. 
Note that this trajectory satisfies Eq.~\ref{eq:w_3} and Eq.~\ref{eq:dif tilde}, and hence satisfies Eq.~\ref{eq:traj2 dif}.
It also satisfies the initial condition $\bw(t_1)$.
Thus, it remains to show that $\bw(t) \in \cb'$ for every $t \in [t_1,\infty)$.
Since by Eq.\ref{eq:traj2} we have $\tilde{\bw}(\infty) = \alpha (5,-1)^\top$, then we have $\bw(\infty) = (5 \alpha ,-\alpha ,w_3(t_1))^\top$. By Lemma~\ref{lemma:in B all gamma}, it follows that $w_3(t_1)$ is in the required interval, and thus it completes the proof of the theorem.

\stam{
Whenever $\bw(t)$ satisfies $\bw^\top(t) \bx_i \geq 0$ for $i \in \{1,2\}$ and $\bw^\top(t) \bx_3 \leq 0$, it follows a trajectory such that the first two components of $\bw(t)$ follow the trajectory from Eq.~\ref{eq:traj2}, and the third component stays at its value in $\bw(t_1)$.
We will show that this condition holds for $t \in [t_1,\infty)$. 
Since by Eq.\ref{eq:traj2} we have $\tilde{\bw}(\infty) = (5,-1)^\top$, then it implies that $\bw(\infty) = (5,-1,s)^\top$, where $s$ is the third component of $\bw(t_1)$, and therefore completes the proof of the theorem.
}

The eigenvalues of $\tilde{X}^\top \tilde{X}$ are $\tilde{\lambda}_1 = 15+5\sqrt{5}$ and $\tilde{\lambda}_2 = 15-5\sqrt{5}$. Let $\tilde{D}=\diag(\tilde{\lambda}_1,\tilde{\lambda}_2)$, and let $\tilde{U} \in \reals^{2 \times 2}$ such that $\tilde{X}^\top \tilde{X} = \tilde{U} \tilde{D} \tilde{U}^\top$.
By Eq.~\ref{eq:traj2}, we have
\begin{equation}
\label{eq:traj2b}
\tilde{\bw}(t) 
= \tilde{U} \exp\left(-(t-t_1) \tilde{D}\right) \tilde{U}^\top \left( \tilde{\bw}(t_1) - \alpha (5,-1)^\top \right) + \alpha (5,-1)^\top~,
\end{equation}
where $\exp\left(-(t-t_1) \tilde{D}\right) = \diag\left(e^{-\tilde{\lambda}_1 (t-t_1)},e^{-\tilde{\lambda}_2 (t-t_1)}\right)$.

\stam{
\begin{lemma}
	\[
	(0,1,0) \bw(t_1) = - (0,0,1) \bw(t_1)~.
	\]
\end{lemma}
\begin{proof}
	\[
	(0,1,0) \bw(t_1) + (0,0,1) \bw(t_1)
	= (0,1,1) \bw(t_1)
	= \frac{1}{\gamma} \bx_3^\top \bw(t_1)
	= 0~.
	\] 
\end{proof}
}

From the following two lemmas it follows that $\bw(t) \in \cb'$ for all $t \geq t_1$.

\begin{lemma}
	For every $t \geq t_1$ and $i \in \{1,2\}$, we have $\bx_i^\top \bw(t) \geq 0$.
\end{lemma}
\begin{proof}
	Since $\bx_1$ and $\bx_2$ have $0$ in their third components, then it suffices to show that $\tilde{\bx}_i^\top \tilde{\bw}(t) \geq 0$ for every $i \in \{1,2\}$ and $t \geq t_1$. 
	By Lemma~\ref{lemma:in B all gamma}, we have $\bw(t_1) = \alpha \cdot (a,b,-b)^\top$, where $a \in (4,5)$ and $b \in (0,1)$.
	
	For every $t \geq t_1$, by straightforward calculations we obtain
	\begin{align*}
		\frac{1}{\alpha} \cdot \tilde{\bx}_1^\top \tilde{\bw}(t) 
		&= (3,-1)  \tilde{U} \exp\left(-(t-t_1) \tilde{D}\right) \tilde{U}^\top (a-5,b+1)^\top  + (3,-1) (5,-1)^\top
		\\
		&= \left(\frac{(3 + \sqrt{5}) a}{2} + \frac{(\sqrt{5}-1) b}{2} - 8 - 2 \sqrt{5} \right) e^{-\tilde{\lambda}_1 (t-t_1)}
		\\
		& \;\;\;\; + \left(\frac{(3-\sqrt{5}) a}{2} - \frac{(\sqrt{5}+1) b}{2}  - 8 + 2 \sqrt{5} \right) e^{-\tilde{\lambda}_2 (t-t_1)} + 16~.
	\end{align*}
	Since $a \in (4,5)$ and $b \in (0,1)$, the above is at least
	\begin{align*}
		&\left(\frac{(3 + \sqrt{5}) \cdot 4}{2} + \frac{(\sqrt{5}-1) \cdot 0}{2} - 8 - 2 \sqrt{5} \right) e^{-\tilde{\lambda}_1 (t-t_1)}
		\\
		& \;\;\;\; + \left(\frac{(3-\sqrt{5}) \cdot 4}{2} - \frac{(\sqrt{5}+1) \cdot 1}{2}  - 8 + 2 \sqrt{5} \right) e^{-\tilde{\lambda}_2 (t-t_1)} + 16
		\\ 
		& = -2 e^{-\tilde{\lambda}_1 (t-t_1)} - \left(2 + \frac{\sqrt{5}+1}{2} \right) e^{-\tilde{\lambda}_2 (t-t_1)} + 16 
		\\ 
		& \geq -2 -2 - \frac{\sqrt{5}+1}{2}  + 16 
		\geq 0~.
	\end{align*}

	Next, we have 
	\begin{align*}
		\frac{1}{\alpha} \cdot \tilde{\bx}_2^\top & \tilde{\bw}(t) 
		= (4,2)  \tilde{U} \exp\left(-(t-t_1) \tilde{D}\right) \tilde{U}^\top (a-5,b+1)^\top  + (4,2) (5,-1)^\top
		\\
		&= \left((\sqrt{5}+2)a + b - 9 - 5\sqrt{5} \right) e^{-\tilde{\lambda}_1 (t-t_1)} +
		\left(-(\sqrt{5}-2)a + b -9 +5\sqrt{5} \right) e^{-\tilde{\lambda}_2 (t-t_1)} + 18
		\\
		&\geq \left((\sqrt{5}+2) \cdot 4 + 0 - 9 - 5\sqrt{5} \right) e^{-\tilde{\lambda}_1 (t-t_1)} +
		\left(-(\sqrt{5}-2) \cdot 5 + 0 -9 +5\sqrt{5} \right) e^{-\tilde{\lambda}_2 (t-t_1)} + 18
		\\
		&= -(\sqrt{5}+1) e^{-\tilde{\lambda}_1 (t-t_1)} + e^{-\tilde{\lambda}_2 (t-t_1)} + 18
		\geq -(\sqrt{5}+1)  + 18 
		\geq 0~.
	\end{align*}
\end{proof}

\begin{lemma}
	For every $t \geq t_1$, we have $\bx_3^\top \bw(t) \leq 0$.
\end{lemma}
\begin{proof}
	By Lemma~\ref{lemma:in B all gamma}, we have $\bw(t_1) = \alpha \cdot (a,b,-b)^\top$, where $a \in (4,5)$ and $b \in (0,1)$.
	The first two components in $\bw(t)$ are $\tilde{\bw}(t)$, and the third component is $-\alpha b$. Hence, in order to show that  
	\[
		\bx_3^\top \bw(t) = (0,\gamma,\gamma) \bw(t) = (0,\gamma) \tilde{\bw}(t) - \alpha b \gamma \leq 0~,
		\]
	it suffices to show that 
	$\frac{1}{\alpha} \cdot (0,1) \tilde{\bw}(t) \leq b$.
	
	We have
	\begin{align}
	\label{eq:x3 off1}
	 	\frac{1}{\alpha} \cdot (0,1) \tilde{\bw}(t) 
	 	&= (0,1)  \tilde{U} \exp\left(-(t-t_1) \tilde{D}\right) \tilde{U}^\top (a-5,b+1)^\top  + (0,1) (5,-1)^\top \nonumber
	 	\\
	 	&= \left(\frac{\sqrt{5}a}{10} + \left(\frac{1}{2} - \frac{\sqrt{5}}{5}\right)b + \frac{1}{2} - \frac{7 \sqrt{5}}{10}  \right) e^{-\tilde{\lambda}_1 (t-t_1)} \nonumber
	 	\\
	 	& \;\;\;\;\;+ \left(-\frac{\sqrt{5}a}{10} + \left( \frac{1}{2} + \frac{\sqrt{5}}{5}\right) b + \frac{1}{2} + \frac{7\sqrt{5}}{10} \right) e^{-\tilde{\lambda}_2 (t-t_1)} - 1~.
	\end{align}
 	Note that the above equals to 
 	\begin{align*}
 		& \frac{\sqrt{5}a}{10} \left(e^{-\tilde{\lambda}_1 (t-t_1)} - e^{-\tilde{\lambda}_2 (t-t_1)} \right) +
 		\left(\left(\frac{1}{2} - \frac{\sqrt{5}}{5}\right)b + \frac{1}{2} - \frac{7 \sqrt{5}}{10}  \right) e^{-\tilde{\lambda}_1 (t-t_1)} 
 		\\
 		& \;\;\;\;\;+ \left(\left( \frac{1}{2} + \frac{\sqrt{5}}{5}\right) b + \frac{1}{2} + \frac{7\sqrt{5}}{10} \right) e^{-\tilde{\lambda}_2 (t-t_1)} - 1~.
 	\end{align*}
 	Since $\tilde{\lambda}_1 > \tilde{\lambda}_2 > 0$, then the above expression is monotonically descreasing w.r.t. $a$. Also, since $a \in (4,5)$, then by plugging in $a=4$ to Eq.~\ref{eq:x3 off1} we obtain
 	\begin{align*}
 		\frac{1}{\alpha} \cdot (0,1) \tilde{\bw}(t) 
 		&\leq \left(-\frac{3\sqrt{5}}{10} + \left(\frac{1}{2} - \frac{\sqrt{5}}{5}\right)b + \frac{1}{2} \right) e^{-\tilde{\lambda}_1 (t-t_1)} \nonumber
 		\\
 		& \;\;\;\;\;+ \left(\frac{3\sqrt{5}}{10} + \left( \frac{1}{2} + \frac{\sqrt{5}}{5}\right) b + \frac{1}{2}\right) e^{-\tilde{\lambda}_2 (t-t_1)} - 1~.
 	\end{align*}
 	Let $g(t)$ denote the r.h.s. of the above inequality. Note that $g(t_1)=b$.
 	Hence, in order to show that $g(t) \leq b$ for every $t \geq t_1$, it suffices to show that $g'(t) \leq 0$ for all $t \geq t_1$.
 	
 	We denote
 	\begin{align*}
 	C_1 &= -\frac{3\sqrt{5}}{10} + \left(\frac{1}{2} - \frac{\sqrt{5}}{5}\right)b + \frac{1}{2}~,
 	\\
 	C_2 &= \frac{3\sqrt{5}}{10} + \left( \frac{1}{2} + \frac{\sqrt{5}}{5}\right) b + \frac{1}{2}~.
 	\end{align*}
    Thus, $g(t) = C_1  e^{-\tilde{\lambda}_1 (t-t_1)} + C_2 e^{-\tilde{\lambda}_2 (t-t_1)} - 1$. Since $b \in (0,1)$, it is easy to verify that $C_1<0$ and $C_2 > 0$. 
    Thus,
    \begin{align*}
    	g'(t) 
    	&= -\tilde{\lambda}_1 C_1 e^{-\tilde{\lambda}_1 (t-t_1)} -\tilde{\lambda}_2  C_2 e^{-\tilde{\lambda}_2 (t-t_1)}
    	\\
    	& \leq -\tilde{\lambda}_1 C_1 e^{-\tilde{\lambda}_1 (t-t_1)} -\tilde{\lambda}_2  C_2 e^{-\tilde{\lambda}_1 (t-t_1)}
    	\\
    	&= \left[-\tilde{\lambda}_1 C_1 - \tilde{\lambda}_2  C_2 \right] e^{-\tilde{\lambda}_1 (t-t_1)}~.
    \end{align*}
    
 	We have 
 	\begin{align*}
 		-\tilde{\lambda}_1 C_1 -\tilde{\lambda}_2  C_2
 		&= - (15+5\sqrt{5}) \left(-\frac{3\sqrt{5}}{10} + \left(\frac{1}{2} - \frac{\sqrt{5}}{5}\right)b + \frac{1}{2}\right)
 		\\
 		&\;\;\;\;\; -(15-5\sqrt{5}) \left(\frac{3\sqrt{5}}{10} + \left( \frac{1}{2} + \frac{\sqrt{5}}{5}\right) b + \frac{1}{2} \right)~.
 	\end{align*}
 	Since $b \geq 0$ then the above is at most
 	\begin{align*}
 		- (15+5\sqrt{5}) \left(-\frac{3\sqrt{5}}{10} +  \frac{1}{2}\right) - (15-5\sqrt{5}) \left(\frac{3\sqrt{5}}{10} + \frac{1}{2} \right)
 		= 0~.
 	\end{align*} 	
\end{proof}

\subsubsection{Proof of Lemma~\ref{lemma:in B all gamma} for $\gamma=1$}
\label{sec:in B gamma 1}

Here, we have $\lambda_1=\frac{27}{2}+\frac{\sqrt{649}}{2}$, $\lambda_2=5$, and $\lambda_3=\frac{27}{2}-\frac{\sqrt{649}}{2}$.
The proof of Lemma~\ref{lemma:in B all gamma} for $\gamma=1$ follows from the following lemmas.

\begin{lemma}
	Let $\bw(t)$ be the trajectory from Eq.~\ref{eq:traj1b}. For $i \in \{1,2\}$, we have $\bx_i^\top \bw(t) \geq 0$ for every $t \geq 0$.
\end{lemma}
\begin{proof}
By straightforward calculations, we obtain 	
\begin{align*}
	\frac{1}{\alpha} \cdot \bx_1^\top \bw(t)
	&= (3,-1,0) (5,-1,1)^\top - (3,-1,0) U \exp(-tD) U^\top (5,-1,1)^\top
	\\
	&= 16 - \left(\frac{58}{9} + \frac{1354\sqrt{649}}{5841} \right) e^{-\lambda_1 t} -  \frac{28}{9} e^{-\lambda_2 t} - \left(\frac{58}{9} - \frac{1354\sqrt{649}}{5841} \right) e^{-\lambda_3 t}
	\\
	&\geq 16 - \left(\frac{58}{9} + \frac{1354\sqrt{649}}{5841} \right) -  \frac{28}{9} - \left(\frac{58}{9} - \frac{1354\sqrt{649}}{5841} \right)
	= 0~. 
\end{align*}

Moreover, we have
\begin{align*}
	\frac{1}{\alpha} \cdot \bx_2^\top \bw(t)
	&= (4,2,0) (5,-1,1)^\top - (4,2,0) U \exp(-tD) U^\top (5,-1,1)^\top
	\\
	&= 18 - \left(\frac{89}{9} + \frac{2357\sqrt{649}}{5841} \right) e^{-\lambda_1 t} + \frac{16}{9}  e^{-\lambda_2 t} + \left(\frac{2357\sqrt{649}}{5841} - \frac{89}{9} \right)  e^{-\lambda_3 t}~. 
\end{align*}
Since $\lambda_1 \geq \lambda_2 \geq \lambda_3 \geq 0$, then the above is at least
\begin{align*}
	18 &- \left(\frac{89}{9} + \frac{2357\sqrt{649}}{5841} \right) e^{-\lambda_1 t} + \frac{16}{9}  e^{-\lambda_1 t} + \left(\frac{2357\sqrt{649}}{5841} - \frac{89}{9} \right)  e^{-\lambda_1 t}
	& = 18 - 18  e^{-\lambda_1 t} \geq 0~. 
\end{align*}
\end{proof}

\begin{lemma}
	\label{lemma:t1 gamma 1}
	Let $\bw(t)$ be the trajectory from Eq.~\ref{eq:traj1b}. There exists $t_1 \in (0.169,0.17)$ such that $\bx_3^\top \bw(t_1)=0$ and for every $t \in [0,t_1]$ we have $\bx_3^\top \bw(t) \geq 0$.
\end{lemma}
\begin{proof}
	We have
	\begin{align*}
		\frac{1}{\alpha} \cdot \bx_3^\top \bw(t)
		&= (0,1,1) (5,-1,1)^\top - (0,1,1) U \exp(-tD) U^\top (5,-1,1)^\top
		\\
		&= - \left(\frac{10}{9} + \frac{10\sqrt{649}}{5841} \right) e^{-\lambda_1 t} + \frac{20}{9}  e^{-\lambda_2 t} - \left(\frac{10}{9} - \frac{10\sqrt{649}}{5841} \right) e^{-\lambda_3 t}~. 
	\end{align*}
	Let 
	\[
	g(t) = - \left(\frac{10}{9} + \frac{10\sqrt{649}}{5841} \right) + \frac{20}{9}  e^{(\lambda_1-\lambda_2) t} - \left(\frac{10}{9} - \frac{10\sqrt{649}}{5841} \right) e^{(\lambda_1-\lambda_3) t}
	\]
	Note that $\frac{1}{\alpha} \cdot \bx_3^\top \bw(t) = e^{-\lambda_1 t} g(t)$. Thus, $\bx_3^\top \bw(t) =0$ iff $g(t)=0$, and $\bx_3^\top \bw(t) \geq 0$ iff $g(t) \geq 0$. 
	We will show that there exists $t_1 \in (0.169,0.17)$ such that $g(t_1)=0$ and for every $t \in [0,t_1]$ we have $g(t) \geq 0$.
	
	Consider the derivative of $g(t)$
	\[
	g'(t) = \frac{20}{9} (\lambda_1-\lambda_2) e^{(\lambda_1-\lambda_2) t} - \left(\frac{10}{9} - \frac{10\sqrt{649}}{5841} \right) (\lambda_1-\lambda_3) e^{(\lambda_1-\lambda_3) t}~.
	\]
	We have $g'(t)=0$ iff 
	\[
	\ln\left(\frac{20}{9} (\lambda_1-\lambda_2)\right) + (\lambda_1-\lambda_2) t = \ln\left( \left(\frac{10}{9} - \frac{10\sqrt{649}}{5841} \right) (\lambda_1-\lambda_3) \right) + (\lambda_1-\lambda_3) t~.
	\]
	Note that this equation has a single solution. Since $g(0)=0$, it follows that $g$ has at most one additional root. 
	It is easy to verify that $g(0.169)>0$ and $g(0.17)<0$. Hence, $g(t)=0$ iff $t \in \{0,t_1\}$ for some $t_1 \in (0.169,0.17)$, and $g(t)>0$ for every $t \in (0,t_1)$.
\end{proof}

\begin{lemma}
	\[
	(1,0,0) \bw(t_1) \in \alpha \cdot (4,5)~.
	\]
\end{lemma}
\begin{proof}
	We have
	\begin{align*}
		\frac{1}{\alpha} \cdot (1,0,0) \bw(t_1)	
		&= (1,0,0) (5,-1,1)^\top - (1,0,0) U \exp(-t_1 D) U^\top (5,-1,1)^\top
		\\
		&= 5 -  \left(\frac{1013 \sqrt{649}}{11682} + \frac{41}{18} \right) e^{-\lambda_1 t_1} - \frac{4}{9} e^{-\lambda_2 t_1} - \left(\frac{41}{18} - \frac{1013 \sqrt{649}}{11682} \right) e^{-\lambda_3 t_1}~. 
	\end{align*}
	
	By Lemma~\ref{lemma:t1 gamma 2} we have $t_1 \in (0.169,0.17)$, and hence 
	\begin{align*}
		\frac{1}{\alpha} \cdot (1,0,0) \bw(t_1)	
		&\geq 5 -  \left(\frac{1013 \sqrt{649}}{11682} + \frac{41}{18} \right) e^{-\lambda_1 \cdot 0.169} - \frac{4}{9}  e^{-\lambda_2 \cdot 0.169} - \left(\frac{41}{18} - \frac{1013 \sqrt{649}}{11682} \right) e^{-\lambda_3 \cdot 0.169}~. 
	\end{align*}
	It is easy to verify that the above expression is greater than $4$.
	
	We also have
	\begin{align*}
		\frac{1}{\alpha} \cdot (1,0,0) \bw(t_1)	
		&\leq 5 -  \left(\frac{1013 \sqrt{649}}{11682} + \frac{41}{18} \right) e^{-\lambda_1 \cdot 1.7} - \frac{4}{9} e^{-\lambda_2 \cdot 1.7} - \left(\frac{41}{18} - \frac{1013 \sqrt{649}}{11682} \right) e^{-\lambda_3 \cdot 1.7} < 5~. 
	\end{align*}	
\end{proof}

\begin{lemma}
	\[
	(0,1,0) \bw(t_1) = - (0,0,1) \bw(t_1)  \in \alpha \cdot (0.035,0.045)~.
	\]
\end{lemma}
\begin{proof}
	We have
	\begin{align*}
		\frac{1}{\alpha} \cdot (0,0,1) \bw(t_1)	
		&= (0,0,1) (5,-1,1)^\top - (0,0,1) U \exp(-t_1 D) U^\top (5,-1,1)^\top
		\\
		&= 1 - \left(\frac{13}{18} - \frac{311 \sqrt{649}}{11682} \right) e^{-\lambda_1 t_1} + \frac{4}{9}  e^{-\lambda_2 t_1} - \left(\frac{13}{18} + \frac{311 \sqrt{649}}{11682} \right) e^{-\lambda_3 t_1}~. 
	\end{align*}
	
	By Lemma~\ref{lemma:t1 gamma 2} we have $t_1 \in (0.169,0.17)$, and hence 
	\begin{align*}
		\frac{1}{\alpha} \cdot (0,0,1) \bw(t_1)	
		&\leq 1 - \left(\frac{13}{18} - \frac{311 \sqrt{649}}{11682} \right) e^{-\lambda_1 \cdot 0.17} + \frac{4}{9}  e^{-\lambda_2 \cdot 0.169} - \left(\frac{13}{18} + \frac{311 \sqrt{649}}{11682} \right) e^{-\lambda_3 \cdot 0.17}~. 
	\end{align*}
	It is easy to verify that the above expression is smaller than $-0.035$.
	
	We also have
	\begin{align*}
		\frac{1}{\alpha} \cdot (0,0,1) \bw(t_1)	
		&\geq 1 - \left(\frac{13}{18} - \frac{311 \sqrt{649}}{11682} \right) e^{-\lambda_1 \cdot 0.169} + \frac{4}{9}  e^{-\lambda_2 \cdot 0.17} - \left(\frac{13}{18} + \frac{311 \sqrt{649}}{11682} \right) e^{-\lambda_3 \cdot 0.169}~. 
	\end{align*}
	It is easy to verify that the above expression is greater than $-0.045$.
	
	Finally, by Lemma~\ref{lemma:t1 gamma 1}, we have
	\[
	0 = \bx_3^\top \bw(t_1) = (0,1,1) \bw(t_1) = (0,1,0) \bw(t_1) + (0,0,1) \bw(t_1)~,
	\]
	and hence $(0,1,0) \bw(t_1) = - (0,0,1) \bw(t_1)$.
\end{proof}

\subsubsection{Proof of Lemma~\ref{lemma:in B all gamma} for $\gamma=2$}
\label{sec:in B gamma 2}

Here, we have $\lambda_1=14+2\sqrt{39}$, $\lambda_2=10$, and $\lambda_3=14-2\sqrt{39}$.
The proof of Lemma~\ref{lemma:in B all gamma} for $\gamma=2$ follows from the following lemmas.

\begin{lemma}
Let $\bw(t)$ be the trajectory from Eq.~\ref{eq:traj1b}. For $i \in \{1,2\}$, we have $\bx_i^\top \bw(t) \geq 0$ for every $t \geq 0$.
\end{lemma}
\begin{proof}
By straightforward calculations, we obtain 	
\begin{align*}
	\frac{1}{\alpha} \cdot \bx_1^\top \bw(t)
	&= (3,-1,0) (5,-1,1)^\top - (3,-1,0) U \exp(-tD) U^\top (5,-1,1)^\top
	\\
	&= 16 - \left( \frac{47}{7} + \frac{17\sqrt{39}}{21}\right) e^{-\lambda_1 t} - \frac{18}{7} e^{-\lambda_2 t} - \left( \frac{47}{7} - \frac{17\sqrt{39}}{21}\right) e^{-\lambda_3 t}
	\\
	&\geq 16 - \left( \frac{47}{7} + \frac{17\sqrt{39}}{21}\right) - \frac{18}{7} - \left( \frac{47}{7} - \frac{17\sqrt{39}}{21}\right)
	= 0~. 
\end{align*}
	
Moreover, we have
\begin{align*}
	\frac{1}{\alpha} \cdot \bx_2^\top \bw(t)
	&= (4,2,0) (5,-1,1)^\top - (4,2,0) U \exp(-tD) U^\top (5,-1,1)^\top
	\\
	&= 18 - \left( \frac{463\sqrt{39}}{273} + \frac{66}{7}  \right) e^{-\lambda_1 t} + \frac{6}{7} e^{-\lambda_2 t} + \left(\frac{463\sqrt{39}}{273} - \frac{66}{7} \right)  e^{-\lambda_3 t}~. 
\end{align*}
Since $\lambda_1 \geq \lambda_2 \geq \lambda_3 \geq 0$, then the above is at least
\begin{align*}
	18 &- \left( \frac{463\sqrt{39}}{273} + \frac{66}{7}  \right) e^{-\lambda_1 t} + \frac{6}{7} e^{-\lambda_1 t} + \left(\frac{463\sqrt{39}}{273} - \frac{66}{7} \right)  e^{-\lambda_1 t}
	\\
	& = 18 - 18  e^{-\lambda_1 t} \geq 0~. 
\end{align*}
\end{proof}

\begin{lemma}
\label{lemma:t1 gamma 2}
	Let $\bw(t)$ be the trajectory from Eq.~\ref{eq:traj1b}. There exists $t_1 \in (0.138,0.139)$ such that $\bx_3^\top \bw(t_1)=0$, and for every $t \in [0,t_1]$ we have $\bx_3^\top \bw(t) \geq 0$.
\end{lemma}
\begin{proof}
	We have
	\begin{align*}
		\frac{1}{\alpha} \cdot \bx_3^\top \bw(t)
		&= (0,2,2) (5,-1,1)^\top - (0,2,2) U \exp(-tD) U^\top (5,-1,1)^\top
		\\
		&= - \left(\frac{15}{7} + \frac{40\sqrt{39}}{273}  \right) e^{-\lambda_1 t} + \frac{30}{7}  e^{-\lambda_2 t} - \left( \frac{15}{7} - \frac{40\sqrt{39}}{273} \right) e^{-\lambda_3 t} ~. 
	\end{align*}
	Let 
	\[
	g(t) = - \left(\frac{15}{7} + \frac{40\sqrt{39}}{273}  \right) + \frac{30}{7}  e^{(\lambda_1-\lambda_2) t} - \left( \frac{15}{7} - \frac{40\sqrt{39}}{273} \right) e^{(\lambda_1-\lambda_3) t}~.
	\]
	Note that $\frac{1}{\alpha} \cdot \bx_3^\top \bw(t) = e^{-\lambda_1 t} g(t)$. Thus, $\bx_3^\top \bw(t) =0$ iff $g(t)=0$, and $\bx_3^\top \bw(t) \geq 0$ iff $g(t) \geq 0$. 
	We will show that there exists $t_1 \in (0.138,0.139)$ such that $g(t_1)=0$ and for every $t \in [0,t_1]$ we have $g(t) \geq 0$.
	
	Consider the derivative of $g(t)$
	\[
	g'(t) = \frac{30}{7} (\lambda_1-\lambda_2) e^{(\lambda_1-\lambda_2) t} - \left( \frac{15}{7} - \frac{40\sqrt{39}}{273} \right) (\lambda_1-\lambda_3) e^{(\lambda_1-\lambda_3) t}~.
	\]
	We have $g'(t)=0$ iff 
	\[
	\ln\left(\frac{30}{7} (\lambda_1-\lambda_2)\right) + (\lambda_1-\lambda_2) t = \ln\left( \left( \frac{15}{7} - \frac{40\sqrt{39}}{273} \right) (\lambda_1-\lambda_3) \right) + (\lambda_1-\lambda_3) t~.
	\]
	Note that this equation has a single solution. Since $g(0)=0$, it follows that $g$ has at most one additional root. 
	It is easy to verify that $g(0.138)>0$ and $g(0.139)<0$. Hence, $g(t)=0$ iff $t \in \{0,t_1\}$ for some $t_1 \in (0.138,0.139)$, and $g(t)>0$ for every $t \in (0,t_1)$.
\end{proof}

\begin{lemma}
	\[
	(1,0,0) \bw(t_1) \in \alpha \cdot (4,5)~.
	\]
\end{lemma}
\begin{proof}
	We have
	\begin{align*}
		\frac{1}{\alpha} \cdot (1,0,0) \bw(t_1)	
		&= (1,0,0) (5,-1,1)^\top - (1,0,0) U \exp(-t_1 D) U^\top (5,-1,1)^\top
		\\
		&= 5 -  \left(\frac{16}{7} + \frac{181\sqrt{39}}{546} \right) e^{-\lambda_1 t_1} - \frac{3}{7}  e^{-\lambda_2 t_1} - \left(\frac{16}{7} - \frac{181\sqrt{39}}{546} \right) e^{-\lambda_3 t_1}~. 
	\end{align*}
	
	By Lemma~\ref{lemma:t1 gamma 2} we have $t_1 \in (0.138,0.139)$, and hence 
	\begin{align*}
		\frac{1}{\alpha} \cdot (1,0,0) \bw(t_1)	
		&\geq 5 -  \left(\frac{16}{7} + \frac{181\sqrt{39}}{546} \right) e^{-\lambda_1 \cdot 0.138} - \frac{3}{7}  e^{-\lambda_2 \cdot 0.138} - \left(\frac{16}{7} - \frac{181\sqrt{39}}{546} \right) e^{-\lambda_3 \cdot 0.138}~. 
	\end{align*}
	It is easy to verify that the above expression is greater than $4$.
	
	We also have
		\begin{align*}
		\frac{1}{\alpha} \cdot (1,0,0) \bw(t_1)	
		&\leq 5 -  \left(\frac{16}{7} + \frac{181\sqrt{39}}{546} \right) e^{-\lambda_1 \cdot 0.139} - \frac{3}{7}  e^{-\lambda_2 \cdot 0.139} - \left(\frac{16}{7} - \frac{181\sqrt{39}}{546} \right) e^{-\lambda_3 \cdot 0.139} < 5~. 
	\end{align*}
\end{proof}

\begin{lemma}
	\[
	(0,1,0) \bw(t_1) = - (0,0,1) \bw(t_1)  \in \alpha \cdot (0.1,0.12)~.
	\]
\end{lemma}
\begin{proof}
	We have
	\begin{align*}
		\frac{1}{\alpha} \cdot (0,0,1) \bw(t_1)	
		&= (0,0,1) (5,-1,1)^\top - (0,0,1) U \exp(-t_1 D) U^\top (5,-1,1)^\top
		\\
		&= 1 - \left(\frac{13}{14} - \frac{61\sqrt{39}}{546} \right) e^{-\lambda_1 t_1} + \frac{6}{7} e^{-\lambda_2 t_1} - \left(\frac{13}{14} + \frac{61\sqrt{39}}{546} \right) e^{-\lambda_3 t_1}~. 
	\end{align*}
	
	By Lemma~\ref{lemma:t1 gamma 2} we have $t_1 \in (0.138,0.139)$, and hence 
	\begin{align*}
		\frac{1}{\alpha} \cdot (0,0,1) \bw(t_1)	
		&\leq 1 - \left(\frac{13}{14} - \frac{61\sqrt{39}}{546} \right) e^{-\lambda_1 \cdot 0.139} + \frac{6}{7} e^{-\lambda_2 \cdot 0.138} - \left(\frac{13}{14} + \frac{61\sqrt{39}}{546}\right) e^{-\lambda_3 \cdot 0.139}~. 
	\end{align*}
	It is easy to verify that the above expression is smaller than $-0.1$.
	
	We also have
	\begin{align*}
		\frac{1}{\alpha} \cdot (0,0,1) \bw(t_1)	
		&\geq 1 - \left(\frac{13}{14} - \frac{61\sqrt{39}}{546} \right) e^{-\lambda_1 \cdot 0.138} + \frac{6}{7} e^{-\lambda_2 \cdot 0.139} - \left(\frac{13}{14} + \frac{61\sqrt{39}}{546}\right) e^{-\lambda_3 \cdot 0.138}~. 
	\end{align*}
	It is easy to verify that the above expression is greater than $-0.12$.
	
	Finally, by Lemma~\ref{lemma:t1 gamma 2}, we have
	\[
		0 = \bx_3^\top \bw(t_1) = 2 \cdot (0,1,1) \bw(t_1) = 2 \left((0,1,0) \bw(t_1) + (0,0,1) \bw(t_1) \right)~,
	\]
	and hence $(0,1,0) \bw(t_1) = - (0,0,1) \bw(t_1)$.
\end{proof}

\subsubsection{Proof of Lemma~\ref{lemma:in B all gamma} for $\gamma=5$}
\label{sec:in B gamma 5}

Here, we have $\lambda_1=\frac{5 \sqrt{105}}{2} + \frac{55}{2}$, $\lambda_2=25$, and $\lambda_3=\frac{55}{2} - \frac{5 \sqrt{105}}{2}$.
The proof of Lemma~\ref{lemma:in B all gamma} for $\gamma=5$ follows from the following lemmas.

\begin{lemma}
	Let $\bw(t)$ be the trajectory from Eq.~\ref{eq:traj1b}. For $i \in \{1,2\}$, we have $\bx_i^\top \bw(t) \geq 0$ for every $t \geq 0$.
\end{lemma}
\begin{proof}
	By straightforward calculations, we obtain 	
	\begin{align*}
		\frac{1}{\alpha} \cdot \bx_1^\top \bw(t)
		&= (3,-1,0) (5,-1,1)^\top - (3,-1,0) U \exp(-tD) U^\top (5,-1,1)^\top
		\\
		&= 16 + \left(\frac{34 \sqrt{105}}{273} - \frac{14}{13}  \right) e^{-\lambda_1 t} -  \frac{180}{13} e^{-\lambda_2 t} -  \left(\frac{34 \sqrt{105}}{273} +\frac{14}{13} \right) e^{-\lambda_3 t}~.
	\end{align*}
	Since $\lambda_1 \geq \lambda_2 \geq \lambda_3 \geq 0$, then the above is at least
	\begin{align*}
		16 + \left(\frac{34 \sqrt{105}}{273} - \frac{14}{13}  \right) e^{-\lambda_1 t} -  \frac{180}{13} e^{-\lambda_3 t} - \left(\frac{34 \sqrt{105}}{273} + \frac{14}{13} \right) e^{-\lambda_3 t}
		\\
		= 16 + \left(\frac{34 \sqrt{105}}{273} - \frac{14}{13}  \right) e^{-\lambda_1 t} - \left( \frac{180}{13} +\frac{34 \sqrt{105}}{273} +  \frac{14}{13}  \right) e^{-\lambda_3 t}~.
	\end{align*}
	Let $g(t)$ denote the above expression. It is easy to verify that $g(0)=0$. Thus, it suffices to show that $g'(t) \geq 0$ for every $t \geq 0$. We have
	\begin{align*}
		g'(t)
		&= -\lambda_1  \left(\frac{34 \sqrt{105}}{273} - \frac{14}{13}  \right) e^{-\lambda_1 t} + \lambda_3 \left( \frac{180}{13} +\frac{34 \sqrt{105}}{273} +  \frac{14}{13}  \right) e^{-\lambda_3 t}
		\\
		&\geq  -\lambda_1  \left(\frac{34 \sqrt{105}}{273} - \frac{14}{13}  \right) e^{-\lambda_1 t} + \lambda_3 \left( \frac{180}{13} +\frac{34 \sqrt{105}}{273} +  \frac{14}{13}  \right) e^{-\lambda_1 t}
		\\
		&= \left[ -\lambda_1  \left(\frac{34 \sqrt{105}}{273} - \frac{14}{13}  \right) + \lambda_3 \left( \frac{180}{13} +\frac{34 \sqrt{105}}{273} +  \frac{14}{13}  \right) \right] e^{-\lambda_1 t}~. 
	\end{align*}
	By plugging in $\lambda_1$ and $\lambda_2$, it is easy to verify that the above expression is positive.

	Next, we have
	\begin{align*}
		\frac{1}{\alpha} \cdot \bx_2^\top \bw(t)
		&= (4,2,0) (5,-1,1)^\top - (4,2,0) U \exp(-tD) U^\top (5,-1,1)^\top
		\\
		&= 18 - \left(\frac{37 \sqrt{105}}{273} - \frac{3}{13} \right) e^{-\lambda_1 t} - \frac{240}{13} e^{-\lambda_2 t} + \left(\frac{37 \sqrt{105}}{273} + \frac{3}{13} \right)  e^{-\lambda_3 t}
		\\
		&\geq 18 - \left(\frac{37 \sqrt{105}}{273} - \frac{3}{13} \right) e^{-\lambda_3 t} - \frac{240}{13} e^{-\lambda_3 t} + \left(\frac{37 \sqrt{105}}{273} + \frac{3}{13} \right)  e^{-\lambda_3 t}
		\\
		&= 18 - 18 e^{-\lambda_3 t}
		\geq 0~. 
	\end{align*}
\end{proof}

\begin{lemma}
	\label{lemma:t1 gamma 5}
	Let $\bw(t)$ be the trajectory from Eq.~\ref{eq:traj1b}. There exists $t_1 \in (0.086,0.087)$ such that $\bx_3^\top \bw(t_1)=0$, and for every $t \in [0,t_1]$ we have $\bx_3^\top \bw(t) \geq 0$.
\end{lemma}
\begin{proof}
	We have
	\begin{align*}
		\frac{1}{\alpha} \cdot \bx_3^\top \bw(t)
		&= (0,5,5) (5,-1,1)^\top - (0,5,5) U \exp(-tD) U^\top (5,-1,1)^\top
		\\
		&= - \left(\frac{46 \sqrt{105}}{273} + \frac{30}{13}  \right)  e^{-\lambda_1 t} + \frac{60}{13}  e^{-\lambda_2 t} - \left(\frac{30}{13} - \frac{46 \sqrt{105}}{273} \right) e^{-\lambda_3 t} ~. 
	\end{align*}
	Let 
	\[
	g(t) = - \left(\frac{46 \sqrt{105}}{273} + \frac{30}{13}  \right) + \frac{60}{13}  e^{(\lambda_1-\lambda_2) t} - \left(\frac{30}{13} - \frac{46 \sqrt{105}}{273} \right) e^{(\lambda_1-\lambda_3) t}~.
	\]
	Note that $\frac{1}{\alpha} \cdot \bx_3^\top \bw(t) = e^{-\lambda_1 t} g(t)$. Thus, $\bx_3^\top \bw(t) =0$ iff $g(t)=0$, and $\bx_3^\top \bw(t) \geq 0$ iff $g(t) \geq 0$. 
	We will show that there exists $t_1 \in (0.086,0.087)$ such that $g(t_1)=0$ and for every $t \in [0,t_1]$ we have $g(t) \geq 0$.
	
	Consider the derivative of $g(t)$
	\[
	g'(t) =  \frac{60}{13} (\lambda_1-\lambda_2) e^{(\lambda_1-\lambda_2) t} -  \left(\frac{30}{13} - \frac{46 \sqrt{105}}{273} \right) (\lambda_1-\lambda_3) e^{(\lambda_1-\lambda_3) t}~.
	\]
	We have $g'(t)=0$ iff 
	\[
	\ln\left( \frac{60}{13} (\lambda_1-\lambda_2)\right) + (\lambda_1-\lambda_2) t = \ln\left( \left(\frac{30}{13} - \frac{46 \sqrt{105}}{273} \right) (\lambda_1-\lambda_3) \right) + (\lambda_1-\lambda_3) t~.
	\]
	Note that this equation has a single solution. Since $g(0)=0$, it follows that $g$ has at most one additional root. 
	It is easy to verify that $g(0.086)>0$ and $g(0.087)<0$. Hence, $g(t)=0$ iff $t \in \{0,t_1\}$ for some $t_1 \in (0.086,0.087)$, and $g(t)>0$ for every $t \in (0,t_1)$.
\end{proof}

\begin{lemma}
	\[
	(1,0,0) \bw(t_1) \in \alpha \cdot (4,5)~.
	\]
\end{lemma}
\begin{proof}
	We have
	\begin{align*}
		\frac{1}{\alpha} \cdot (1,0,0) \bw(t_1)	
		&= (1,0,0) (5,-1,1)^\top - (1,0,0) U \exp(-t_1 D) U^\top (5,-1,1)^\top
		\\
		&= 5 -  \left(\frac{5}{26} - \frac{31 \sqrt{105}}{2730} \right) e^{-\lambda_1 t_1} - \frac{60}{13}  e^{-\lambda_2 t_1} - \left(\frac{5}{26}+ \frac{31 \sqrt{105}}{2730} \right) e^{-\lambda_3 t_1}~. 
	\end{align*}
	
	By Lemma~\ref{lemma:t1 gamma 5} we have $t_1 \in (0.086,0.087)$, and hence 
	\begin{align*}
		\frac{1}{\alpha} \cdot (1,0,0) \bw(t_1)	
		&\geq 5 -  \left(\frac{5}{26} - \frac{31 \sqrt{105}}{2730} \right) e^{-\lambda_1 \cdot 0.086} - \frac{60}{13}   e^{-\lambda_2 \cdot 0.086} - \left(\frac{5}{26} + \frac{31 \sqrt{105}}{2730} \right) e^{-\lambda_3 \cdot 0.086}~. 
	\end{align*}
	It is easy to verify that the above expression is greater than $4$.
	
	We also have
	\begin{align*}
		\frac{1}{\alpha} \cdot (1,0,0) \bw(t_1)	
		&\geq 5 -  \left(\frac{5}{26} - \frac{31 \sqrt{105}}{2730} \right) e^{-\lambda_1 \cdot 0.087} - \frac{60}{13}   e^{-\lambda_2 \cdot 0.087} - \left(\frac{5}{26} + \frac{31 \sqrt{105}}{2730} \right) e^{-\lambda_3 \cdot 0.087}
		<5~. 
	\end{align*}
\end{proof}

\begin{lemma}
	\[
	(0,1,0) \bw(t_1) = - (0,0,1) \bw(t_1)  \in \alpha \cdot (0.2,0.22)~.
	\]
\end{lemma}
\begin{proof}
	We have
	\begin{align*}
		\frac{1}{\alpha} \cdot (0,0,1) \bw(t_1)	
		&= (0,0,1) (5,-1,1)^\top - (0,0,1) U \exp(-t_1 D) U^\top (5,-1,1)^\top
		\\
		&= 1 - \left(\frac{25}{26} - \frac{31 \sqrt{105}}{546} \right) e^{-\lambda_1 t_1} + \frac{12}{13} e^{-\lambda_2 t_1} - \left(\frac{25}{26} + \frac{31 \sqrt{105}}{546} \right) e^{-\lambda_3 t_1}~. 
	\end{align*}
	
	By Lemma~\ref{lemma:t1 gamma 5} we have $t_1 \in (0.086,0.087)$, and hence 
	\begin{align*}
		\frac{1}{\alpha} \cdot (0,0,1) \bw(t_1)	
		&\leq 1 - \left(\frac{25}{26} - \frac{31 \sqrt{105}}{546} \right) e^{-\lambda_1 \cdot 0.087} + \frac{12}{13} e^{-\lambda_2 \cdot 0.086} - \left(\frac{25}{26} + \frac{31 \sqrt{105}}{546} \right) e^{-\lambda_3 \cdot 0.087}~. 
	\end{align*}
	It is easy to verify that the above expression is smaller than $-0.2$.
	
	We also have
	\begin{align*}
		\frac{1}{\alpha} \cdot (0,0,1) \bw(t_1)	
		&\geq 1 - \left(\frac{25}{26} - \frac{31 \sqrt{105}}{546} \right) e^{-\lambda_1 \cdot 0.086} + \frac{12}{13} e^{-\lambda_2 \cdot 0.087} - \left(\frac{25}{26} + \frac{31 \sqrt{105}}{546} \right) e^{-\lambda_3 \cdot 0.086}~. 
	\end{align*}
	It is easy to verify that the above expression is greater than $-0.22$.
	
	Finally, by Lemma~\ref{lemma:t1 gamma 5}, we have
	\[
	0 = \bx_3^\top \bw(t_1) = 5 \cdot (0,1,1) \bw(t_1) = 5 \left((0,1,0) \bw(t_1) + (0,0,1) \bw(t_1) \right)~,
	\]
	and hence $(0,1,0) \bw(t_1) = - (0,0,1) \bw(t_1)$.
\end{proof}

\section{Proof of Lemma~\ref{lemma:R constant}}
\label{app:R constant}

\begin{lemma}
	\label{lemma:radial}
	The function $\calr$ is radial. That is, there exists a function $f:\reals_+ \rightarrow \reals$ such that $\calr(\bw) = f(\norm{\bw})$ for every $\bw \in \reals^d$.
\end{lemma}
\begin{proof}
	Let $\bw,\bw' \in \reals^d \setminus \{\zero\}$ and $\alpha>0$ such that $\norm{\bw} = \norm{\bw'} =  \alpha q$. We denote $r=\norm{\bw}$.
	Let $c' = \frac{c}{q}$, and let $\epsilon = \norm{\bw - \bw'}$. 
	Assume w.l.o.g. that $\epsilon \leq \frac{2r c'}{\sqrt{1+(c')^2}}$. We will show that $\calr(\bw)=\calr(\bw')$. 
	
	Let $\bu = \frac{\bw + \bw'}{2} \left(1 +  \frac{\epsilon^2}{4r^2-\epsilon^2}\right)$, 
	and let $\be \in \reals^d$ be a vector orthogonal to $\bw,\bw'$, such that  
	$\norm{\be}^2 = r^2 \left((c')^2 - \frac{\epsilon^2}{4r^2 - \epsilon^2} \right)$.
	Note that 
	\[
		 4r^2-\epsilon^2
		 \geq 4r^2 - \frac{4r^2 (c')^2}{1+(c')^2}
		 = \frac{4r^2}{1+(c')^2}
		 \geq \frac{\epsilon^2}{(c')^2}~.
	\]
	Thus, $(c')^2 \geq \frac{\epsilon^2}{4r^2 - \epsilon^2}$. 
	Hence $r^2 \left((c')^2 - \frac{\epsilon^2}{4r^2 - \epsilon^2} \right) \geq 0$, and therefore $\be$ is well-defined.
	
	Let $\bu^* = \bu + \be$. We will show that there is a vector $\bv$ such that $\bu^* = \bw + \bv$, $\norm{\bv} = \alpha c$, and $\inner{\bw,\bv} = 0$. It implies that $\calr(\bu^*) = \calr(\bw)$. Likewise, we will show that there is a vector $\bv'$ such that $\bu^* = \bw' + \bv'$, $\norm{\bv'} = \alpha c$ and $\inner{\bw',\bv'} = 0$. It implies that $\calr(\bu^*) = \calr(\bw')$, and hence $\calr(\bw) = \calr(\bw')$, as required.
	
	We first show that $\norm{\bu^* - \bw}=\alpha c$. We have 
	\begin{align*}
		\norm{\bu^* - \bw}^2 = \norm{\bu + \be - \bw}^2~.	
	\end{align*}  
	Since $\be$ is orthogonal to $\bw,\bw'$, it is also orthogonal to $\bu$, and thus the above equals
	\begin{equation}
	\label{eq:ustar minus w}
		\norm{\be}^2 + \norm{\bu - \bw}^2
		= r^2 \left((c')^2 - \frac{\epsilon^2}{4r^2 - \epsilon^2} \right) + \norm{\bu - \bw}^2~.
	\end{equation}
	Moreover,
	\begin{align*}
		\norm{\bu - \bw}^2
		= \norm{\frac{\bw + \bw'}{2} \left(1 +  \frac{\epsilon^2}{4r^2-\epsilon^2}\right) - \bw}^2
		= \norm{\frac{\bw' - \bw}{2} + \frac{\bw' + \bw}{2} \cdot \frac{\epsilon^2}{4r^2-\epsilon^2}}^2~.
	\end{align*}
	Since $\left< \frac{\bw' - \bw}{2}, \frac{\bw' + \bw}{2}  \right> = 0$, the above equals
	\begin{align}
	\label{eq:u minus w part 2}
		\norm{\frac{\bw' - \bw}{2}}^2 + \norm{\frac{\bw' + \bw}{2} \cdot \frac{\epsilon^2}{4r^2-\epsilon^2}}^2 
		= \frac{\epsilon^2}{4} + \left(\frac{\epsilon^2}{4r^2-\epsilon^2}\right)^2 \norm{\frac{\bw' + \bw}{2}}^2~. 
	\end{align}
	Note that 
	\begin{equation}
	\label{eq:inner ww'}
		\epsilon^2 = \norm{\bw-\bw'}^2 = 2r^2 - 2 \cdot \inner{\bw,\bw'}~,
	\end{equation} 
	and hence
	\begin{equation*}
		\label{eq:bw plus bw'}
		\norm{\bw + \bw'}^2 = 2r^2 + 2 \cdot \inner{\bw,\bw'} = 4r^2 - \epsilon^2~.
	\end{equation*} 
 	Combining the above with Eq.~\ref{eq:u minus w part 2} we have 
 	\begin{align*}
 	\label{eq:u minus w part 3}
 	 	\norm{\bu - \bw}^2
 	 	&= \frac{\epsilon^2}{4} + \left(\frac{\epsilon^2}{4r^2-\epsilon^2}\right)^2 \cdot \frac{1}{4} \cdot  (4r^2 - \epsilon^2)
 	 	= \frac{1}{4} \left(\epsilon^2 + \frac{\epsilon^4}{4r^2-\epsilon^2} \right)
 	 	\\
 	 	&= \frac{r^2 \epsilon^2}{4r^2-\epsilon^2}~.
 	\end{align*}
  	Plugging the above into Eq.~\ref{eq:ustar minus w}, we have
  	\begin{align*}
  			\norm{\bu^* - \bw}^2 = 
  			r^2 \left((c')^2 - \frac{\epsilon^2}{4r^2 - \epsilon^2} \right) + \frac{r^2 \epsilon^2}{4r^2-\epsilon^2}
  			= (rc')^2 
  			= \left(\frac{rc}{q}\right)^2
  			= (\alpha c)^2~.
  	\end{align*}
  
  	We now show that $\inner{\bw,\bu^*-\bw}=0$. We have 
  	\[
  		\inner{\bw,\bu^*-\bw}
  		= \inner{\bw,\bu + \be - \bw}~.
  	\]
  	Since $\be$ is orthogonal to $\bw$, the above equals
  	\begin{align*}
  		\inner{\bw,\bu-\bw}
  		&= \inner{\bw,\bu} - r^2
  		\\
  		&= \frac{\inner{\bw,\bw + \bw'}}{2} \left(1 +  \frac{\epsilon^2}{4r^2-\epsilon^2}\right) - r^2
  		\\
  		&= \left(r^2 + \inner{\bw,\bw'}\right) \cdot \frac{1}{2} \left(1 +  \frac{\epsilon^2}{4r^2-\epsilon^2}\right) - r^2~.
  	\end{align*}
    By Eq.~\ref{eq:inner ww'}, the above equals 
    \[
    	\left(2r^2 - \frac{\epsilon^2}{2}\right) \cdot \frac{1}{2} \left(1 +  \frac{\epsilon^2}{4r^2-\epsilon^2}\right) - r^2
    	= \frac{4r^2 - \epsilon^2}{4} \cdot\frac{4r^2}{4r^2-\epsilon^2} - r^2
    	= 0~.
    \] 
  	
  	Finally, the proof that $\norm{\bu^*-\bw'}=\alpha c$ and that $\inner{\bw',\bu^*-\bw'}=0$ is similar.
\end{proof}

By Lemma~\ref{lemma:radial}, we have $\calr(\bw)=f(\norm{\bw})$ for every $\bw \in \reals^d$, for some $f:\reals_+ \rightarrow \reals$.
We will show that for every $0 < r < r'$ we have $f(r)=f(r')$.  Let $\bw \in \reals^d$ such that $\norm{\bw}=r$, and let $\alpha = \frac{r}{q}$. By our assumption on $\calr$, for a vector $\bv \in \reals^d$ such that $\norm{\bv} = \alpha c$ and $\inner{\bw,\bv}=0$, we have $\calr(\bw) = \calr(\bw+\bv) = \min_{\beta \geq 0}\calr(\bw + \beta \bv)$. 

Since $\calr(\bw) = \calr(\bw+\bv)$ then we have $f(r) = f(\norm{\bw}) = f(\norm{\bw+\bv})$.
Note that 
\[
	\norm{\bw + \bv} 
	= \sqrt{\norm{\bw}^2+\norm{\bv}^2}
	= \sqrt{r^2+(\alpha c)^2}
	= r \sqrt{1+\frac{c^2}{q^2}}~.
\]
Hence, we have $f(r) = f\left(r \sqrt{1+\frac{c^2}{q^2}}\right)$. By repeating the above steps, we have for every $m \in \nat$ that $f(r) = f\left(r \left(\sqrt{1+\frac{c^2}{q^2}}\right)^m \right)$.
Let $m'$ be a sufficiently large integer, such that $r \left(\sqrt{1+\frac{c^2}{q^2}}\right)^{m'} > r'$.

Since $\calr(\bw) = \min_{\beta \geq 0}\calr(\bw + \beta \bv)$ then we have $f(r) = \min_{r'' \geq r} f(r'') \leq f(r')$. By repeating this argument with $\bw' \in \reals^d$ such that $\norm{\bw'}=r'$ we also have $f(r') \leq f\left(r \left(\sqrt{1+\frac{c^2}{q^2}}\right)^{m'}\right)$. 

Overall, we have 
\[
	f(r) \leq f(r') \leq f\left(r \left(\sqrt{1+\frac{c^2}{q^2}}\right)^{m'}\right) = f(r)~,
\]
and thus $f(r)=f(r')$ as required.


\section{Proofs for Section~\ref{sec:depth 2}}

\subsection{Proof of Lemma~\ref{lemma:v_t}}
\label{app:proof of lemma v_t}

We have
\begin{align*}
	&\frac{\mathop{d}}{\mathop{d t}} \left( \norm{\bw(t)}^2 - (v(t))^2 \right)
	\\
	&= 2 \bw(t)^\top \dot{\bw}(t) - 2 v(t) \dot{v}(t) 
	\\
	&= -2 \bw(t)^\top \frac{\partial \cl_{X,\by}(\btheta(t))}{\partial \bw} + 2 v(t)  \frac{\partial \cl_{X,\by}(\btheta(t))}{\partial v}	
	\\
	&= -2 \bw(t)^\top \sum_{i=1}^n \left(v(t) \sigma(\bx_i^\top \bw(t)) - y_i \right) v(t) \sigma'(\bx_i^\top \bw(t)) \bx_i + 2 v(t) \sum_{i=1}^n \left( v(t) \sigma(\bx_i^\top \bw(t)) - y_i \right) \sigma(\bx_i^\top \bw(t))
	\\
	&= -2  \sum_{i=1}^n \left(v(t) \sigma(\bx_i^\top \bw(t)) - y_i \right) v(t) \sigma'(\bx_i^\top \bw(t)) \bw(t)^\top \bx_i + 2 \sum_{i=1}^n \left( v(t) \sigma(\bx_i^\top \bw(t)) - y_i \right) v(t) \sigma(\bx_i^\top \bw(t))~.
\end{align*}
Since the ReLU function $\sigma$ satisfies $\sigma'(\bx_i^\top \bw(t)) (\bx_i^\top \bw(t)) = \sigma(\bx_i^\top \bw(t))$, then the above equals $0$.

\subsection{Proof of Lemma~\ref{lemma:rotation and scaling}}
\label{app:proof of lemma rotation and scaling}

For the input $(X,\by)$, gradient flow obeys the following dynamics.
\begin{equation}
	\label{eq:derivative w}
	\dot{\bw}(t) 
	= -\frac{\partial \cl_{X,\by}((\bw(t),v(t)))}{\partial \bw} 
	= - \sum_{i=1}^n \left( v(t) \sigma(\bx_i^\top \bw(t))-y_i \right) v(t) \sigma'(\bx_i^\top \bw(t)) \bx_i~,
\end{equation}
and 
\begin{equation}
	\label{eq:derivative v}
	\dot{v}(t) = -\frac{\partial \cl_{X,\by}((\bw(t),v(t)))}{\partial v} = - \sum_{i=1}^n \left( v(t) \sigma(\bx_i^\top \bw(t))-y_i \right) \sigma(\bx_i^\top \bw(t))~.
\end{equation}

\subsubsection{Proof of part (1)}

Let $\epsilon>0$ be a constant, and let $\btheta^\epsilon(t) = (\bw(t),v(t))$ be the trajectory of gradient flow with the input $(X,\by)$ starting from $\btheta^\epsilon(0) = (\zero,\epsilon)$.
Let $\tilde{\bw}(t) = M \bw(t)$ and $\tilde{v}(t) = v(t)$. 
By Eq.~\ref{eq:derivative w} we have
\begin{align*}
	\dot{\tilde{\bw}}(t) 
	&= M \dot{\bw}(t) 
	= -M \sum_{i=1}^n \left( v(t) \sigma(\bx_i^\top \bw(t))-y_i \right) v(t) \sigma'(\bx_i^\top \bw(t)) \bx_i
	\\
	&=  -M \sum_{i=1}^n \left( v(t) \sigma((M \bx_i)^\top (M \bw(t)))-y_i \right) v(t) \sigma'((M \bx_i)^\top (M \bw(t))) \bx_i	
	\\
	&= - \sum_{i=1}^n \left( \tilde{v}(t) \sigma((M \bx_i)^\top \tilde{\bw}(t))-y_i \right) \tilde{v}(t) \sigma'((M \bx_i)^\top \tilde{\bw}(t)) (M \bx_i)
	\\
	&=  -\frac{\partial \cl_{XM^\top,\by}((\tilde{\bw}(t),\tilde{v}(t)))}{\partial \bw}~.
\end{align*}
Moreover, by Eq.~\ref{eq:derivative v} we have
\begin{align*}
	\dot{\tilde{v}}(t)
	&= \dot{v}(t)
	= - \sum_{i=1}^n \left( v(t) \sigma(\bx_i^\top \bw(t))-y_i \right) \sigma(\bx_i^\top \bw(t))
	\\
	&= - \sum_{i=1}^n \left( \tilde{v}(t) \sigma((M \bx_i)^\top \tilde{\bw}(t))-y_i \right) \sigma((M \bx_i)^\top \tilde{\bw}(t))
	\\
	&= -\frac{\partial \cl_{XM^\top,\by}((\tilde{\bw}(t),\tilde{v}(t)))}{\partial v}~.
\end{align*}
Hence, $\tilde{\btheta}^\epsilon(t) = (\tilde{\bw}(t),\tilde{v}(t))$ is the trajectory of gradient flow with the input $(XM^\top,\by)$ staring from $\tilde{\btheta}^\epsilon(0) = (\tilde{\bw}(0),\tilde{v}(0)) = (\zero,\epsilon)$. Therefore, 
\[
\lim_{\epsilon \rightarrow 0^+} \tilde{\btheta}^\epsilon(\infty) 
= \lim_{\epsilon \rightarrow 0^+} (M \bw(\infty),v(\infty))
= (M \bw^*, v^*)~.
\]

\subsubsection{Proof of part (2)}

Let $\epsilon>0$ be a constant, and let $\btheta^{\epsilon/\sqrt{\alpha}}(t) = (\bw(t),v(t))$ be the trajectory of gradient flow with the input $(X,\by)$ starting from $\btheta^{\epsilon/\sqrt{\alpha}}(0) = \left(\zero,\frac{\epsilon}{\sqrt{\alpha}}\right)$.
Let $\tilde{\bw}(t) = \sqrt{\alpha} \bw(\alpha t)$ and $\tilde{v}(t) = \sqrt{\alpha} v(\alpha t)$. 
By Eq.~\ref{eq:derivative w} we have
\begin{align*}
	\dot{\tilde{\bw}}(t)
	&= \sqrt{\alpha} \cdot \alpha \cdot \dot{\bw}(\alpha t)
	\\
	&= - \alpha \sqrt{\alpha}  \sum_{i=1}^n \left( v(\alpha t) \sigma(\bx_i^\top \bw(\alpha t))-y_i \right) v(\alpha t) \sigma'(\bx_i^\top \bw(\alpha t)) \bx_i
	\\	
	&= - \sum_{i=1}^n \left( \sqrt{\alpha} v(\alpha t) \sigma(\bx_i^\top \sqrt{\alpha} \bw(\alpha t))- (\alpha y_i) \right) (\sqrt{\alpha} v(\alpha t)) \sigma'(\bx_i^\top \bw(\alpha t)) \bx_i ~.
\end{align*}
Since $\sigma'(\bx_i^\top \bw(\alpha t)) = \sigma'(\bx_i^\top \sqrt{\alpha}\bw(\alpha t))$, the above equals
\begin{align*}
	- \sum_{i=1}^n \left( \tilde{v}(t) \sigma(\bx_i^\top \tilde{\bw}(t))- (\alpha y_i) \right) \tilde{v}(t) \sigma'(\bx_i^\top \tilde{\bw}(t)) \bx_i 
	=  -\frac{\partial \cl_{X, \alpha \by}((\tilde{\bw}(t),\tilde{v}(t)))}{\partial \bw}~.
\end{align*}
Moreover, by Eq.~\ref{eq:derivative v} we have
\begin{align*}
	\dot{\tilde{v}}(t)
	&= \sqrt{\alpha} \cdot \alpha \cdot \dot{v}(\alpha t)
	\\
	&= - \alpha \sqrt{\alpha}	\sum_{i=1}^n \left( v(\alpha t) \sigma(\bx_i^\top \bw(\alpha t))-y_i \right) \sigma(\bx_i^\top \bw(\alpha t))
	\\
	&= - \sum_{i=1}^n \left( \sqrt{\alpha} v(\alpha t) \sigma(\bx_i^\top  \sqrt{\alpha} \bw(\alpha t)) - (\alpha y_i) \right) \sigma(\bx_i^\top \sqrt{\alpha} \bw(\alpha t))
	\\
	&= - \sum_{i=1}^n \left( \tilde{v}(t) \sigma(\bx_i^\top  \tilde{\bw}(t)) - (\alpha y_i) \right) \sigma(\bx_i^\top \tilde{\bw}(t))
	\\
	&= -\frac{\partial \cl_{X, \alpha \by}((\tilde{\bw}(t),\tilde{v}(t)))}{\partial v}~.
\end{align*}
Hence, $\tilde{\btheta}^\epsilon(t) = (\tilde{\bw}(t),\tilde{v}(t))$ is the trajectory of gradient flow with the input $(X, \alpha \by)$ staring from $\tilde{\btheta}^\epsilon(0) = (\tilde{\bw}(0),\tilde{v}(0)) = \left(\zero,\sqrt{\alpha} \cdot \frac{\epsilon}{\sqrt{\alpha}}\right) = (\zero, \epsilon)$. Therefore,
\[
\lim_{\epsilon \rightarrow 0^+} \tilde{\btheta}^\epsilon(\infty) 
= \lim_{\epsilon \rightarrow 0^+} (\sqrt{\alpha} \bw(\infty), \sqrt{\alpha} v(\infty))
= \sqrt{\alpha} \cdot(\bw^*, v^*)~.
\]	

\subsection{Proof of Corollary~\ref{cor:R' properties}}
\label{app:proof of cor R' properties}

Let $X_0,X_5 \in \reals^{3 \times 3}$ and $\by \in \reals^3$ from Corollary~\ref{cor:from ass1}, such that for the corresponding limit points $\btheta^*,\tilde{\btheta}^*$, and vectors $\bu^*=\Psi^{-1}(\btheta^*)$, $\tilde{\bu}^*=\Psi^{-1}(\tilde{\btheta}^*)$, and $\bu' =  \tilde{\bu}^* - \bu^*$, we have $\norm{\bu'} > 0$ and $\inner{\bu^*,\bu'}=0$. 
Furthermore, for every $\beta \geq 0$ we have $\bu^* + \beta \bu' \in \cu_{X_0,\by} \cap \cu_{X_5,\by}$. 
We denote $c=\norm{\bu'}$.
Let $\bu \in \reals^3$ and $\alpha>0$ such that $\norm{\bu} = \alpha \cdot \norm{\bu^*}$, and let $\bu_\bot \in \reals^3$ such that $\inner{\bu,\bu_\bot}=0$ and $\norm{\bu_\bot} = \alpha c$. 

Let $M \in SO(3)$ be a rotaion matrix such that $\alpha M \bu^* = \bu$ and $\alpha M \bu' = \bu_\bot$.
Let $\btheta^*_{\alpha M}$ and $\tilde{\btheta}^*_{\alpha M}$ be the limit points of gradient flow with inputs $(X_0 M^\top,\alpha \by)$ and $(X_5 M^\top,\alpha \by)$, respectively. 
By Lemma~\ref{lemma:rotation and scaling}, we have $\Psi^{-1}(\btheta^*_{\alpha M}) = \alpha M \bu^* = \bu$, and 
\[
\Psi^{-1}(\tilde{\btheta}^*_{\alpha M}) = \alpha M \tilde{\bu}^* = \alpha M (\bu^* + \bu') = \alpha M \bu^* + \alpha M \bu' = \bu + \bu_\bot~.
\]
Let $\{\bx_i^\top\}_{i=1}^3$ and $\{\tilde{\bx}_i^\top\}_{i=1}^3$ be the rows of $X_0$ and $X_5$ (respectively).
Since for every $\beta \geq 0$ we have $\bu^* + \beta \bu' \in \cu_{X_0,\by} \cap \cu_{X_5,\by}$, then for every $1 \leq i \leq 3$ and $\bx \in \{\bx_i,\tilde{\bx}_i\}$, and every $\beta \geq 0$, we have 
\begin{align*}
	\alpha y_i
	&= \alpha \sigma(\bx^\top (\bu^* + \beta \bu'))
	= \alpha \sigma((M\bx)^\top (M\bu^* + M \beta \bu'))
	= \sigma((M\bx)^\top (\alpha M\bu^* + \alpha M \beta \bu'))
	\\
	&= \sigma((M\bx)^\top (\bu + \beta \bu_\bot)).
\end{align*}
Hence, $\bu + \beta \bu_\bot \in \cu_{X_0 M^\top, \alpha \by} \cap \cu_{X_5 M^\top, \alpha \by}$.
Since $\bu \in \argmin_{\bar{\bu} \in  \cu_{X_0 M^\top, \alpha \by}} \calr'(\bar{\bu})$ and $\bu + \bu_\bot \in \argmin_{\bar{\bu} \in  \cu_{X_5 M^\top, \alpha \by}} \calr'(\bar{\bu})$, then we have
$\calr'(\bu) = \calr'(\bu + \bu_\bot) = \min_{\beta \geq 0}\calr'(\bu + \beta \bu_\bot)$. 

\end{document}